\documentclass{article}

\usepackage[final]{neurips_2025}

\usepackage[utf8]{inputenc} % allow utf-8 input
\usepackage[T1]{fontenc}    % use 8-bit T1 fonts
\usepackage{hyperref}       % hyperlinks
\usepackage{url}            % simple URL typesetting
\usepackage{booktabs}       % professional-quality tables
\usepackage{amsfonts}       % blackboard math symbols
\usepackage{nicefrac}       % compact symbols for 1/2, etc.
\usepackage{microtype}      % microtypography
\usepackage{xcolor}         % colors

\usepackage{subcaption}         % Allows for the use of subfigures and subcaptions
\usepackage{graphicx}

\usepackage{algorithm}
\usepackage{algorithmic}

\usepackage{amssymb}            % Defines common symbols like \mathbb R
\usepackage{amsthm}
\newtheorem{theorem}{Theorem}
\newtheorem{proposition}{Proposition}

\newtheorem{defn}{Definition}

\usepackage{amsmath}            % Defines common symbols like \mathbb R

\usepackage{tikz}
\usetikzlibrary{arrows.meta}
\usetikzlibrary{patterns.meta}
\usetikzlibrary{patterns}

\usetikzlibrary{decorations.pathreplacing}% for "show path construction"
\tikzset{
    cheating dash/.code args={on #1 off #2 ends #3}{%
        \csname tikz@addoption\endcsname{%Use csname so catcode of @ doesn't have do be changed.
            \pgfgetpath\currentpath%
            \pgfprocessround{\currentpath}{\currentpath}%
            \csname pgf@decorate@parsesoftpath\endcsname{\currentpath}{\currentpath}%
            \pgfmathparse{max(#1-#3,0)}\let\dashphase=\pgfmathresult%
            \pgfmathparse{\csname pgf@decorate@totalpathlength\endcsname-#1+2*\dashphase}\let\rest=\pgfmathresult%
            \pgfmathparse{#1+#2}\let\onoff=\pgfmathresult%
            \pgfmathparse{max(floor(\rest/\onoff), 1)}\let\nfullonoff=\pgfmathresult%
            \pgfmathparse{max((\rest-\onoff*\nfullonoff)/\nfullonoff+#2, #2)}\let\offexpand=\pgfmathresult%
            \pgfsetdash{{#1}{\offexpand}}{\dashphase pt}}%
    },
	cheating dash per segment/.style args={on #1 off #2 ends #3}{
	    /utils/exec=\csname tikz@options\endcsname,%inherit options/.code={[\csname tikz@options\endcsname]},inherit options,
		decoration={show path construction,
			lineto code={\draw [cheating dash=on #1 off #2 ends #3] (\tikzinputsegmentfirst) -- (\tikzinputsegmentlast);},
			curveto code={\draw [cheating dash=on #1 off #2 ends #3] (\tikzinputsegmentfirst) .. controls (\tikzinputsegmentsupporta) and (\tikzinputsegmentsupportb) .. (\tikzinputsegmentlast);},
			closepath code={\draw [cheating dash=on #1 off #2 ends #3] (\tikzinputsegmentfirst) -- (\tikzinputsegmentlast);}
		},
		decorate,
	},
}

\definecolor{myblue}{HTML}{0072B2}
\definecolor{mygreen}{HTML}{275D38}

\hypersetup{
    pdftitle={The World Is Bigger! A Computationally-Embedded Perspective on the Big World Hypothesis},
    pdfauthor={Alex Lewandowski, Aditya A. Ramesh, Edan Meyer, Dale Schuurmans, Marlos C. Machado},
}

\title{The World Is Bigger!\\ A Computationally-Embedded Perspective\\ on the Big World Hypothesis}

\author{%
  Alex Lewandowski$^{1,2,\star}$\quad
  Aditya A. Ramesh$^{3}$ \quad
  Edan Meyer$^{1,2}$\\
  \textbf{Dale Schuurmans}$^{1,2,4,5}$ \quad
  \textbf{Marlos C. Machado}$^{1,2,4}$ \\
  $^1$University of Alberta \quad
  $^2$Amii \quad
  $^3$The Swiss AI Lab IDSIA, USI \& SUPSI\\
  $^4$Canada CIFAR AI Chair \quad
  $^5$Google DeepMind
}

\begin{document}

\maketitle

\begin{abstract}
    Continual learning is often motivated by the idea, known as the big world hypothesis, that ``the world is bigger'' than the agent.
  Recent problem formulations capture this idea by explicitly constraining an agent relative to the environment.
  These constraints lead to solutions in which the agent continually adapts to best use its limited capacity, rather than converging to a fixed solution.
  However, explicit constraints can be ad hoc, difficult to incorporate, and may limit the effectiveness of scaling up the agent's capacity.
  In this paper, we characterize a problem setting in which an agent, regardless of its capacity, is constrained by being embedded in the environment.
  In particular, we introduce a \emph{computationally-embedded} perspective that represents an embedded agent as an automaton simulated within a universal (formal) computer.
  Such an automaton is always constrained; we prove that it is equivalent to an agent that interacts with a partially observable Markov decision process over a countably infinite state-space.
  We propose an objective for this setting, which we call \emph{interactivity}, that measures an agent's ability to continually adapt its behaviour by learning new predictions.
  We then develop a model-based reinforcement learning algorithm for interactivity-seeking, and use it to construct a synthetic problem to evaluate continual learning capability. 
  Our results show that deep nonlinear networks struggle to sustain interactivity, whereas deep linear networks sustain higher interactivity as capacity increases. 

%%% Local Variables:
%%% TeX-master: "../main.tex"
%%% TeX-command-extra-options: "-shell-escape"
%%% End:

\end{abstract}

{
  \let\thefootnote\relax
  \footnotetext{$^\star$Correspondence to: Alex Lewandowski <lewandowski@ualberta.ca>.}
}

\section{Introduction}
\label{sec:intro}

The goal of this paper is to characterize a general problem setting in which the best use of an agent's limited capacity is to continually adapt~\citep{ring94_contin,thrun98_lifel_learn_algor,abel23}.
Our approach is motivated by the idea, known as the big world hypothesis, that ``the world is bigger'' than the agent~\citep{javed24_big_world_hypot_ramif_artif_intel}.
That is, an agent in a big world may lack the capacity to learn the fixed optimal solution, and should instead continually adapt to new experience by updating its approximate solution~\citep[\emph{i.e.}, by tracking, ][]{sutton07}.
However, formalizing the relationship between the agent and the environment presents a challenge, because they are typically treated as separate entities in reinforcement learning (see Figures~\ref{fig:traditional}~and~\ref{fig:aixi}).
We address this by introducing a computationally-embedded agent that is simulated by the environment's dynamics, thus constraining the agent within the environment.
With this perspective, we construct a problem setting in which any such agent is (i) constrained by its capacity, and (ii) suboptimal if it stops learning.
\looseness=-1

Explicit constraints on the agent have been previously considered in continual learning as a means of capturing the big world hypothesis.
For example, in continual learning experiments, the agent's learning algorithm is often explicitly constrained by what it can store~\citep{prabhu20_gdumb}, or by the expressivity of its function approximator~\citep{meyer24_harnes_discr_repres_contin_reinf_learn}.
Other real-world constraints on the agent's hardware have also been considered.
Such constraints include limits on the agent's compute \citep[see discussion on measuring compute in Section 4.1,][]{verwimp24_contin_learn} and on the agent's energy use \citep{javed24_big_world_hypot_ramif_artif_intel}.
One recent formalization of continual learning uses an explicit constraint on the agent's information-theoretic capacity \citep{kumar23_contin_learn_comput_const_reinf_learn,kumar24_need_big_world_simul}.
However, beyond analytically tractable agents and environments, this information-theoretic constraint is difficult to measure and enforce, limiting its generality as a problem setting for continual learning.
Furthermore, explicitly constraining the agent limits the effectiveness of scaling up the agent's capacity, which has been a major source of progress in machine learning more broadly
\citep{hestness2017deep, kaplan2020scaling, hoffmann22_train_comput_optim_large_languag_model}.
These limitations suggest that explicit constraints may not be an effective way of capturing the big world hypothesis.
\looseness=-1

In contrast to explicit constraints, our approach considers the implicit constraint that arises from an agent embedded in an environment (see Figure \ref{fig:embedded}).
Specifically, an embedded agent is, in principle, fully defined within the environment's state, and simulated by the environment's dynamics.
The embedded nature of agents---that they exist within and as part of their environment---is typically treated as outside the scope of the problem formulation \citep{demski19_embed_agenc}.
However, the physical world is a clear example of a world bigger than any agent, suggesting that embedded agency can provide a natural formalization of the big world hypothesis and of continual learning.

To formalize an agent embedded in an environment, we define a \emph{universal-local environment}: a Markov process whose transition dynamics can simulate an agent within a finite portion of its state-space. 
\emph{Computational universality} guarantees that the environment can simulate any algorithm \citep{church36,turing37_entsc}.
\emph{Uniform locality} decomposes the environment's transition dynamics into a collection of identical Markov processes, each operating on a finite portion of the state-space. 
An embedded agent is then represented as an automaton simulated within the environment, formalized as one of these local Markov processes.
When the automaton's dynamics are determined solely by its input and output, we prove it interacts with a partially observable Markov decision process: it receives inputs (observations), updates its internal state, and produces outputs (actions).
We then propose \emph{interactivity}, defined through algorithmic complexity~\citep{Kolmogorov:65, Solomonoff:64, Chaitin:66}, to measure an automaton's ability to continually adapt its future input-output behaviour.
An interactivity-seeking agent pursues behaviour that is increasingly complex while remaining predictable given its past experience.
Crucially, we prove that an agent's interactivity is constrained by its capacity, which can be allocated either to increase its behavioural complexity or to improve its behavioural predictability.

\begin{figure}[t]
  \centering
  \begin{subfigure}{0.32\textwidth}
    \vfill
    \centering
    \raisebox{1.5mm}{%  % Adjust the 5mm value as needed
      \begin{tikzpicture}
        \begin{scope}[scale=0.6,yshift=1cm]
        \begin{scope}[xshift=0cm]
          \node[above,align=center] at (2.0,2.8) {Bounded\\Env};
          \draw[line width=2pt] (1.0,0.5) rectangle +(2,2);
          \draw[solid, Triangle-Triangle, line width=1pt] (1.1, 0.6) -- (2.9, 2.4);
        \end{scope}
        \begin{scope}[xshift=4cm]
          \node[above] at (1.5,3.2) {\raisebox{0pt}[0pt][0pt]{Scalable Agent}};
          \draw[line width=2pt] (0,0) rectangle (3,3);
          \draw[dashed, Triangle-Triangle, line width=1pt] (0.1, 0.1) -- (2.9, 2.9);
        \end{scope}

        \draw[->,>=latex, line width=2pt] (3,1.75) -- (4,1.75);
        \draw[->,>=latex, line width=2pt] (4,1.25) -- (3,1.25);
        \end{scope}
      \end{tikzpicture}
    }
    \caption{Traditional RL}
    \label{fig:traditional}
    \vfill
  \end{subfigure}
  \begin{subfigure}{0.32\linewidth}
    \vfill
    \centering
    % \resizebox{0.9\linewidth}{!}{
\begin{tikzpicture}[mylines/.style={ultra thick,line cap=rect,line join=square}]
        % First tikzpicture code here
        \begin{scope}[scale=0.6]
        \begin{scope}[xshift=0cm]
          \node[above,align=center] at (1.75,3.5) {\raisebox{0pt}[0pt][0pt]{Universal Env}};
          \draw[mylines, line width=2pt, cheating dash per segment=on 5pt off 3.5pt ends 7.5pt] (0,0) rectangle (3.5,3.5);
        \end{scope}
        \begin{scope}[xshift=4.5cm]
          \node[above,align=center] at (1.75,3.5) {\raisebox{0pt}[0pt][0pt]{AIXI Agent}};
          \draw[mylines, line width=2pt, cheating dash per segment=on 5pt off 3.5pt ends 7.5pt] (0,0) rectangle (3.5,3.5);
        \end{scope}
        \draw[->,>=latex, line width=2pt] (3.5,2.02) -- (4.5,2.02);
        \draw[->,>=latex, line width=2pt] (4.5,1.48) -- (3.5,1.48);
        \end{scope}
      \end{tikzpicture}
    % }
    % \vspace{0mm}
    \caption{Universal AI}
    \label{fig:aixi}
    \vfill
  \end{subfigure}
    \hspace{2mm}
  \begin{subfigure}{0.32\linewidth}
    \vfill
    \centering
    % \resizebox{0.8\linewidth}{!}{
\begin{tikzpicture}[mylines/.style={ultra thick,line cap=rect,line join=square}]
        \begin{scope}[scale=0.6]
        \begin{scope}[xshift=0cm]
          \node[above] at (2.5,3.5) {Universal-Local Env};
          \draw[mylines, line width=2pt, cheating dash per segment=on 5pt off 3.5pt ends 7.5pt] (0,0) rectangle +(5,3.5);

        \end{scope}
        \begin{scope}[xshift=-0.2cm]
          \node[above, align=center, scale=0.7] at (3.5,2.75) {Embedded Agent};
          \draw[line width=2pt] (2.5,0.75) rectangle +(2,2);
          \draw[dashed, Triangle-Triangle, line width=1pt] (2.6, 0.85) -- (4.4, 2.65);
        % \draw[->,>=latex, line width=2pt] (1.5,1.0) -- (0.5,1.0);
        % \draw[<-,>=latex, line width=2pt] (1.5,1.5) -- (0.5,1.5);
        \draw[->,>=latex, line width=2pt] (1.5,2) -- (2.5,2);
        \draw[->,>=latex, line width=2pt] (2.5,1.5) -- (1.5,1.5);
        \end{scope}
        \end{scope}
      \end{tikzpicture}
    % }
    \caption{This work}
    \label{fig:embedded}
    \vfill
  \end{subfigure}

%%% Local Variables:
%%% TeX-master: "../neurips.tex"
%%% TeX-command-extra-options: "-shell-escape"
%%% jinx-local-words: "Env center contin xshift yshift"
%%% End:
\caption{\textbf{Comparing the agent and environment in different problem formulations.} Each problem formulation differs in the constraints that it imposes on the agent relative to the environment.
\textbf{Traditional RL:} A given environment typically has a bounded capacity, but a scalable agent can increase its capacity. Such an agent is unconstrained in principle because its capacity can always be scaled beyond the environment.
\textbf{Universal AI:} Both the computationally universal environment and the AIXI agent are unbounded. AIXI is thus unconstrained but not computable.
\textbf{This work:} The universal-local environment is unbounded, but it can simulate an embedded agent of any bounded capacity within its state-space. An embedded agent is implicitly constrained because the environment necessarily has greater capacity than any agent contained within it.
}
    \vspace{-3mm}
\end{figure}
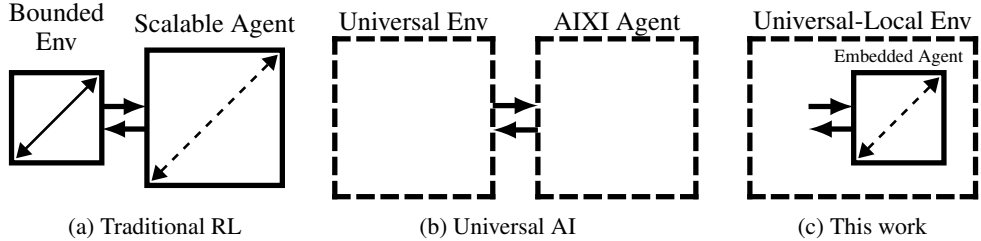

Universal artificial intelligence similarly considers universal environments \citep{hutter2000theory,Hutter:05uaibook}, in which the uncomputable AIXI agent was extended to an embedded formulation~\citep{orseau12_space_time}.
However, these works did not consider the problem of learning under limited capacity.
By considering a universal environment that is also local, we show that an interactivity-seeking agent must continually adapt to its experience, regardless of its capacity.
Interactivity also relates to intrinsic motivation objectives \citep{schmidhuber91_curious,chentanez04_intrin, still12},
such as forecasting complexity~\citep{grassberger86_towar}, statistical complexity~\citep{crutchfield89_infer}, predictive information~\citep{bialek01_predic}, and light cone complexity~\citep{shalizi04_quant,aaronson14_quant_rise_fall_compl_closed_system}.
These objectives are defined as Shannon information theoretic measures of past-future dependence.
Unlike Shannon information, which requires probability distributions, interactivity uses algorithmic information, allowing it to operate directly on individual sequences.
This sequence-based formulation provides a natural computational framework for continual learning, where an agent's behaviour is an individual sequence.
Moreover, by tying the agent's objective to its computational capacity, which is naturally constrained relative to the environment, our formalism further captures the big world hypothesis. 
\looseness=-1

Lastly, we develop a reinforcement learning algorithm for maximizing interactivity, and use it to construct a task that evaluates a learning algorithm's capability for continually adaptive behaviour. 
Specifically, we operationalize algorithmic complexity in terms of the prediction error.
Interactivity is then measured as the reduction in this error attributable to continual learning, relative to a baseline agent that stops learning. 
Interactivity-seeking behaviour thus involves learning a policy to steer the agent's behaviour to new experiences that are learnable, but that would have high prediction error without learning.
We show that interactivity-seeking agents create their own non-stationarity by changing their policy, thereby satisfying key desiderata of the continual learning problem: every agent is constrained by its capacity, and any agent that stops learning is suboptimal.
Our results indicate that, in this setting, deep nonlinear networks struggle to sustain interactivity, whereas deep linear networks effectively scale their interactivity with increased capacity.

%%% Local Variables:
%%% TeX-master: "../main.tex"
%%% TeX-command-extra-options: "-shell-escape"
%%% jinx-local-words: "Env center contin xshift yshift"
%%% End:

\vspace{-1mm}
\section{Background}
\vspace{-1mm}
\label{sec:background}

We formalize an embedded agent through a computational perspective, viewing the environment as a computation that also simulates the agent.
Specifically, we consider a computationally universal environment that simulates any algorithm by implementing each computational step through its state transitions.
Our approach is general, by making use of the Church-Turing thesis, which asserts that all computationally universal systems are equivalent in what they can simulate, and that any such system can simulate another~\citep{church36,turing37_entsc}.
This allows us to adopt a particular model of computation (\emph{e.g.}, Turing machines) while retaining a general class of environments that are capable of simulating an embedded agent.
\looseness=-1

We characterize the capabilities of an agent, relative to its environment, in terms of its input-output behaviour as a finite sequence (\emph{i.e.}, a string).
In particular, we use the algorithmic complexity of a string, which is the length of the shortest program that computes it and halts~\citep{Kolmogorov:65, Solomonoff:64, Chaitin:66}.
\begin{defn}[Algorithmic Complexity]
  Given strings $x,y \in \Sigma^*$, where $\Sigma$ is a finite symbol-set and $\Sigma^*$ is the set of strings, 
  the conditional algorithmic complexity is the length of the shortest program, $|c|$, that halts and outputs $x$ given $y$ as input,
$$\mathbb{K}_{\,\mathcal{U}}(x|y) := \min \{ |c| : \mathcal{U}(c,y) = x\},$$ where $\mathcal{U}$ is a reference universal machine. The unconditional algorithmic complexity is given by $\mathbb{K}_{\,\mathcal{U}}(x) := \mathbb{K}_{\,\mathcal{U}}(x\,|\, \epsilon)$, where $\epsilon$ is the empty string.
\end{defn}
While algorithmic complexity depends on the choice of a reference universal machine, any specific choice affects the algorithmic complexity by, at most, an additive constant independent of the specific string~\citep{li19_introd_kolmog_compl_its_applic_edition}.
Moreover, in this work, we will consider the given computationally universal environment, defineed in the next section, as the canonical reference universal machine.
% This is because, by the Church-Turing thesis, any universal Turing machine can simulate another (\emph{e.g.}, via a compiler).

%%% Local Variables:
%%% TeX-master: "../main.tex"
%%% TeX-command-extra-options: "-shell-escape"
%%% End:

\vspace{-2mm}
\section{A Universal-Local Environment}
\vspace{-2mm}
\label{sec:env}
We begin by defining \emph{universal-local environments}, a general class of environments in which an agent can be embedded. 
These environments are characterized by two key properties: computational universality (Section \ref{sec:universal}), which establishes that any algorithm can be simulated by the environment's transition dynamics, and uniform locality (Section \ref{sec:local}), which ensures that any finite computation can be confined to a bounded portion of the environment's state-space.
In Section~\ref{sec:agent}, we use these properties to define an embedded agent as an automaton simulated on the environment's state-space.

\subsection{Markov Representation of a Computationally Universal Environment}
\label{sec:universal}

We first connect computational processes with Markov processes defined over a countable state-space with a transition function that is computable in polynomial-time with respect to the size of the state.
\begin{defn}
  An algorithmic Markov process, $\mathcal{E} = (\Omega, \Xi, \mathbb{T})$, is a discrete process defined on 
  a countable 
  % an infinite
state-space,
  ${\Omega := \{\omega: \Xi \to \Sigma : |\omega|  < \infty\}}$,
  where $\Sigma$ is a finite symbol set with distinguished blank symbol $\square \in \Sigma$,
  $\Xi$ is a countable set used for indexing,
  and
  ${|\omega| := |\{\xi \in \Xi : \omega(\xi) \neq \square\}|}$ counts the number of non-blank symbols.
  Given an initial state $\omega \in \Omega$ with $|\omega|<\infty$, the process produces the next state $\omega' = \mathbb{T}(\omega)$, such that 
  $|\omega'|<\infty$ and
  where the transition function, $\mathbb{T}: \Omega \rightarrow \Omega$, is computable in time $O(\texttt{poly}(|\omega|))$.  
\end{defn}
The significance of this formulation is that any Turing machine can be represented as an algorithmic Markov process.
Specifically, the Markov state represents the entire configuration of the Turing machine, including its head position, current tape contents, and control state.
The Markov transition function represents the Turing machine's transition function, which is a lookup table that can be applied to the Markov state in polynomial-time.
We will consider only deterministic transitions, but non-deterministic transitions are also possible, and would involve multiple possible next states.
\begin{proposition}[Representing Turing machines]
  \label{prop:markov_env}
  The computational process followed by a Turing machine can be represented as an algorithmic Markov process.
\end{proposition}
All proofs of propositions and theorems can be found in Section \ref{sec:proofs} of the Appendix.
One consequence of Proposition \ref{prop:markov_env} is that there exists a universal Markov process (an algorithmic Markov process corresponding to a universal Turing machine).
This result highlights how an algorithmic Markov process is more general than the Markov processes typically considered in reinforcement learning.
In particular, a universal Markov process is capable of simulating any algorithm, which is crucial for defining an embedded agent in Section~\ref{sec:agent}.

\subsection{Defining Uniform Locality with Boundaried Markov Processes}
\label{sec:local}

Intuitively, uniform locality means that the transition function can be decomposed into identical local transition functions with dynamics that are determined by a finite portion of the state-space.
To make this precise, we first formalize a finite portion of the state-space as a \emph{substate-space}.
\begin{defn}
  Given an algorithmic Markov process, $\mathcal{E} = (\Omega, \Xi, \mathbb{T})$,
  a substate-space, $\Omega|_{F}$, is defined as a restriction of the state-space,  $\Omega$, to a finite index-set, ${F \in \mathcal{F}(\Xi) := \{I : I \subset \Xi, |I| < \infty\}}$: 
$\Omega|_F = \{\omega|_F : \omega \in \Omega\}$  where $\omega|_F : F \to \Sigma$  is defined by, $\omega|_F(\xi) := \omega(\xi) \text{ for all } \xi \in F$.
\end{defn}
We now consider the transition function restricted to a substate-space, ${\Omega|_F}$.
In particular,
we define a boundaried Markov process in which the restricted transition function,  $\mathbb{T}|_{F}$, depends on both the substate-space, $\Omega|_{F}$, and an additional substate-space, $\Omega|_{b(F)}$, called the boundary-space.
\begin{defn}
  Given an algorithmic Markov process, $\mathcal{E} = (\Omega, \Xi, \mathbb{T})$,
  a substate-space, $\Omega|_F$, admits a $k$-horizon boundaried Markov process if there exists a substate-space $\Omega|_{b^k(F)}$ (referred to as a boundary-space) and a restricted transition function $\mathbb{T}|_F^k : \Omega|_F \times \Omega|_{b^k(F)} \to \Omega|_F$
  that is equivalent to the $k$-step transition function on the substate-space: $\mathbb{T}|^k_F(\omega|_F, \omega|_{b^k(F)}) = \mathbb{T}^{(k)}(\omega)|_F$ for all $\omega \in \Omega$.
We denote the $k$-horizon boundaried Markov process as $\mathcal{E}|_{F}^k = (\Omega|_F, \Omega|_{b^k(F)}, \mathbb{T}|_F^k)$, and, if $k=1$, we refer to it simply as a boundaried Markov process, ${\mathcal{E}|_{F} = (\Omega|_F, \Omega|_{b(F)}, \mathbb{T}|_F)}$.
\end{defn}
Typically, the size of the boundary-space increases as the transition horizon $k$ becomes larger. 
This is because the current substate, $\omega|_F \in \Omega|_F$, and the current boundary, $\omega|_{b(F)} \in \Omega|_{b(F)}$, only determine the next substate, $\omega|_F' \in \Omega|_F$, and not the next boundary.

An algorithmic Markov process is uniformly local if every index admits a boundaried Markov process that is isomorphic to a single reference boundaried Markov process.
\begin{defn}[Uniform Locality]
  An algorithmic Markov process, $\mathcal{E} = (\Omega, \Xi, \mathbb{T})$, is uniformly local if there exists a boundary function, ${b: \Xi \to \mathcal{F}(\Xi)}$
  and a reference $\xi_\star \in \Xi$, such that every $\xi \in \Xi$ admits a boundaried Markov process isomorphic to the reference $\xi_\star$. That is, there exist bijections
  $\alpha_\xi: \Omega|_{\{\xi\}} \to \Omega|_{\{\xi_\star\}}$ and $\beta_\xi:  \Omega|_{b(\{\xi\})} \to \Omega|_{b(\{\xi_\star\})}$
  satisfying, for all $\omega \in \Omega$,
  $${\alpha_\xi \left(\mathbb{T}|_{\{\xi\}}(\omega|_{\{\xi\}}, \omega|_{b(\{\xi\})})\right) = \mathbb{T}|_{\{\xi_\star\}}(\alpha_\xi(\omega|_{\{\xi\}}), \beta_\xi(\omega|_{b(\{\xi\})}))}.$$
  \vspace{2mm}
\end{defn}
Uniform locality guarantees that the transition function can be computed by simultaneously applying an identical local transition function to each singleton of the state-space.
Consequently, the transition dynamics for any substate-space $\Omega|_F$ can be determined from the collective boundary, defined by $b(F) = \bigcup_{\xi \in F} b(\{\xi\})$.
This ensures that every substate-space admits a boundaried Markov process.

Thus, we use the term \emph{universal-local environment} for a universal Markov process that is also uniformly local.
We call it an environment because, as we will show, an embedded agent can be simulated as a boundaried Markov process within it.

\subsection{Example of a Universal-Local Environment: Conway's Game of Life}

Conway's Game of Life (or Life) is an example of a universal-local environment \citep{conway70}.
This environment is computationally universal because it can simulate a universal Turing machine \citep{berlekamp82_winnin_ways_your_mathem_plays_vol,rendell11_conway}.
A substate-space in Life is a finite subset of cells on the grid and the possible values taken by those cells.
Furthermore, Life is uniformly local because the one-step transition dynamics for each individual cell is determined by its $8$ adjacent neighbouring cells, which defines the boundary-space (see Figure \ref{fig:life}).
\looseness=-1

We point out Life as an existence proof for universal-local environments, but we are not suggesting to program an agent within it.
Instead, we use our formalism to characterize the constraints that would be faced by an agent if it were embedded in any such environment.

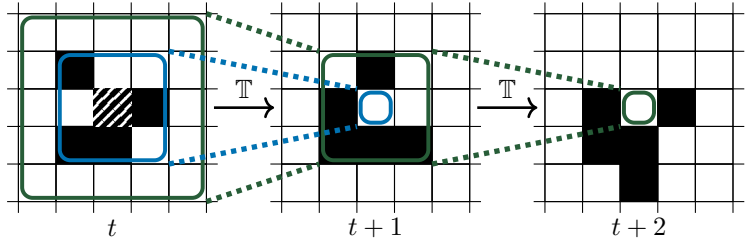
\begin{figure}
  \centering
  \begin{tikzpicture}[scale=0.5]
  % Grid 2
  \begin{scope}[xshift=14cm]

\foreach \x in {0,...,5}
    \draw (\x,-0.3) -- (\x,5.3);
% Draw horizontal lines
\foreach \y in {0,...,5}
    \draw (-0.3,\y) -- (5.3,\y);
    \foreach \x in {0,...,4} \foreach \y in {0,...,4}
      \draw (\x,\y) grid +(1,1);
    \fill (3,2) rectangle +(1,1);
    \fill (1,2) rectangle +(1,1);
    \fill (1,1) rectangle +(1,1);
    \fill (2,1) rectangle +(1,1);
    \fill (2,0) rectangle +(1,1);
    \node at (2.5,-0.7) {$t+2$};

    \draw[mygreen, rounded corners, line width=1.5pt] (2.1,2.1) rectangle +(0.8,0.8);

  \end{scope}

  % Grid 1
  \begin{scope}[xshift=7cm]

\foreach \x in {0,...,5}
    \draw (\x,-0.3) -- (\x,5.3);
% Draw horizontal lines
\foreach \y in {0,...,5}
    \draw (-0.3,\y) -- (5.3,\y);
  \foreach \x in {0,...,4} \foreach \y in {0,...,4}
    \draw (\x,\y) grid +(1,1);
  \fill (2,3) rectangle +(1,1);
  \fill (1,2) rectangle +(1,1);
  \fill (1,1) rectangle +(1,1);
  \fill (2,1) rectangle +(1,1);
  \fill (3,1) rectangle +(1,1);
  \node at (2.5,-0.7) {$t+1$};

      \draw[myblue, rounded corners, line width=1.5pt] (2.1,2.1) rectangle +(0.8,0.8);
      \draw[mygreen, rounded corners, line width=1.5pt] (1.1,1.1) rectangle +(2.8,2.8);
   \draw[dotted, mygreen, line width=2pt] (4,4) -- (9,3);
   \draw[dotted, mygreen, line width=2pt] (4,1) -- (9,2);
  \end{scope}

  % Grid 0
  \begin{scope}[xshift=0cm]

\foreach \x in {0,...,5}
    \draw (\x,-0.3) -- (\x,5.3);
% Draw horizontal lines
\foreach \y in {0,...,5}
    \draw (-0.3,\y) -- (5.3,\y);

    \foreach \x in {0,...,4} \foreach \y in {0,...,4}
      \draw (\x,\y) grid +(1,1);
      \fill (1,3) rectangle +(1,1);
      \fill (2,2) rectangle +(1,1);
      \fill (3,2) rectangle +(1,1);
      \fill (1,1) rectangle +(1,1);
      \fill (2,1) rectangle +(1,1);
      \node at (2.5,-0.7) {$t$};

      % Add highlight around neighborhood
      \draw[myblue, rounded corners, line width=1.5pt] (1.1,1.1) rectangle +(2.8,2.8);
      \draw[mygreen, rounded corners, line width=1.5pt] (0.1,0.1) rectangle +(4.8,4.8);
      \fill[pattern={Lines[angle=45,distance=3pt,line width=1pt]}, pattern color=white] (2,2) rectangle +(1,1);
   \draw[dotted, mygreen, line width=2pt] (5,5) -- (8,4);
   \draw[dotted, mygreen, line width=2pt] (5,0) -- (8,1);
   \draw[dotted, myblue, line width=2pt] (4,4) -- (9,3);
   \draw[dotted, myblue, line width=2pt] (4,1) -- (9,2);
  \end{scope}
  \draw[->, line width = 1pt] (12.2,2.5) -- (13.8,2.5);
  \node[above] at (13,2.5){$\mathbb{T}$};

  % Arrow
  \draw[->, line width = 1pt] (5.2,2.5) -- (6.8,2.5);
  \node[above] at (6,2.5){$\mathbb{T}$};

\end{tikzpicture}

\caption{
  \textbf{Conway's Game of Life is a cellular automaton and an example of a universal-local environment.}
  The state-space is an infinite 2D grid, $\Xi = \mathbb{Z}^2$, where each cell in the grid takes one of two values, $\Sigma = \{\text{black}, \text{white}\}$.
  Every cell uses the same local transition function in which the boundary-space consists of its $8$ adjacent neighbours. 
  The blue and green borders (\emph{left}) correspond to $1$ and $2$ horizon boundary-spaces, which determine the center cell at time-steps $t+1$ (\emph{middle}) and $t+2$ (\emph{right}), respectively. Longer-term transition dynamics depend on larger boundary-spaces.
}
  \label{fig:life}
  \vspace{-12mm}
\end{figure}

%%% Local Variables:
%%% TeX-master: "../main.tex"
%%% TeX-command-extra-options: "-shell-escape"
%%% jinx-local-words: "cd color myblue mygreen tikz xshift"
%%% End:

\vspace{-2mm}
\section{A Computationally-Embedded Agent}
\vspace{-1mm}
\label{sec:agent}
\vspace{-1mm}

We represent an embedded agent as an automaton simulated within the universal-local environment, formalized as a boundaried Markov process.
Specifically, we prove that when the automaton's boundary space coincides with its input-output space, it constitutes a formal agent-environment boundary~\citep{jiang19_value_funct_agent_envir_bound,harutyunyan20_what}. 
Consequently, the automaton is equivalent to a stateful policy interacting with a partially observable Markov decision process.
We then introduce \emph{interactivity} to measure a capability for continually adaptive behaviour, proving that interactivity-seeking automata are constrained by their computational capacity.

\vspace{-2mm}
\subsection{Embedding an Automaton in a Universal-Local Environment}
\vspace{-1mm}

A universal-local environment can simulate any algorithm; this property enables us to define an automaton, $\mathcal{A}$, within the environment's state space, $\Omega$, such that its operation is simulated by the environment's transition dynamics, $\mathbb{T}$.
Moreover, uniform locality ensures that the automaton can be represented as a boundaried Markov process (see Figure \ref{fig:agents}, left).

\begin{defn}
  Given a universal-local environment, $\mathcal{E} = (\Omega, \Xi, \mathbb{T})$,
  an embedded automaton is defined by $\mathcal{A} := (\Omega|_X, \Omega|_Y, \Omega|_\Theta, u, \pi)$, where $\Omega|_\Theta$ is the automaton's internal state space, $\Omega|_{X}, \Omega|_{Y}$ are input and output spaces, $\pi: \Omega|_X \times \Omega|_{\Theta} \to \Omega|_Y$ is the output function, and $u : \Omega|_X \times \Omega|_{\Theta} \to \Omega|_\Theta$ is the automaton's internal state update function. %Note that the automaton's state may also include a finite history of past input and output. 
\end{defn}
Relating this to an agent in reinforcement learning, we may think of the inputs as observations, the internal state as the parameters of a function approximator, the outputs as actions, the internal state update function as a learning rule, and the output function as a policy. The input-space may also provide an external reward to the automaton, but this need not be the case.
\begin{proposition}[Automaton-Environment Relationship]
  \label{prop:agent}
  Given an embedded automaton, 
${\mathcal{A} = (\Omega|_X, \Omega|_Y, \Omega|_\Theta, u, \pi)}$, 
if there exists a horizon $k$ such that $b^k(\Theta) = X$, then:
  \begin{enumerate}
          \vspace{-1mm}
    \item The automaton is equivalent to a $k$-horizon boundaried Markov process
          \vspace{-1mm}
    \item The automaton's environment is a partially observable 
          Markov decision process.
          \vspace{-1mm}
    \item The automaton's interaction is equivalent to a stateful policy acting on the environment
          \vspace{-1mm}
  \end{enumerate}
\end{proposition}
Now that we have defined both the embedded automaton and its environment within the same universal-local environment, we can show that any such automaton is constrained by the size of its internal state-space, which determines its memory and computational capacity.\looseness=-1
\begin{proposition}[Implicitly Constrained]
  \label{prop:implicit}
  The capacity of an embedded automaton is upper bounded by the size of its internal state space, $|\Theta|$, which is finite. Thus, there exists input-output sequences that the automaton cannot realize.
\end{proposition}
An embedded automaton is constrained by its capacity, which constrains the behaviours that it can produce.
While more complex behaviour can be produced with more capacity, we will show that this constraint specifically limits its ability to produce behaviour that continually adapts to past experience.
\looseness=-1

\begin{figure}
  \centering
    \begin{tikzpicture}[scale=0.75, line width=1pt]
    \begin{scope}[xshift=0cm]

      % \draw[dotted] (3.0,0.4) -- (7.0,1.0);
      \draw[dotted] (4.6,0.4) -- (10.2,1.0);
      % \draw[dotted] (4.6,2.0) -- (10.2,4.2);
      \draw[dotted] (3.0,2.0) -- (7.0,4.2);

      \fill[pattern=grid, pattern color=gray!50] (8.8, 1.0) rectangle +(1.4, 3.2);

      \fill[pattern=grid, pattern color=gray!50] (7.8, 2.2) rectangle +(1.0, 0.9);
      \fill[pattern=grid, pattern color=gray!50] (5.9, 1.3) rectangle +(1.0, 1.0);
      \fill[pattern=grid, pattern color=gray!50] (5.9, 2.9) rectangle +(1.0, 1.0);
      % \draw[-] (7.0,1.0) rectangle +(3.2,3.2);
      \draw[-, line width=6pt] (7, 1.7) -- (7,1.0) -- (10,1.0) -- (10,4.2) -- (7.0,4.2) --  (7.0,3.5);% -- (7.0,3.5) -- (7.0,1.9);

      \draw[-, line width=6pt] (7, 1.9) -- (7.0,3.3);

      \draw[->] (6.6,3.4) -- (7.6,3.4);
      \draw[->] (7.2,1.8) -- (6.6,1.8);

      \node at (6.25,3.1) {$\Omega|_{X}$};
      \node at (6.25,1.5) {$\Omega|_{Y}$};

      \draw[-] (8.6,2.2) -- (8.6,1.8);
      \draw[->] (8.6,1.8) -- (7.4,1.8);

      \draw[-] (8.6,3.4) -- (7.8,3.4);
      \draw[->] (8.6,3.4) -- (8.6,3.0);

      \node at (8.0,3.75) {$u$};
      \node at (7.6,1.4) {$\pi$};
      \node at (8.6,2.6) {$\Omega|_{\Theta}$};

      \draw[-] (7.9,2.6) -- (7.7,2.6);
      \draw[->] (7.7,2.6) -- (7.7,3.3);
      \draw[-] (7.6,3.3) rectangle +(0.2, 0.2);
      \draw[-] (7.2,1.7) rectangle +(0.2, 0.2);

      \filldraw[fill=black] (7.3,3.4) circle (0.05);
      \draw[->] (7.3,3.4) -- (7.3,1.9);
  
      \node at (2.5,5.5) {$\omega_{t} \in \Omega$};
      \foreach \x in {0,...,25} {
      \draw (0.2*\x,0.0) -- (0.2*\x,5);
      \draw[-] (0.2*\x,-0.12) -- (0.2*\x,0.0);
      \draw[-] (0.2*\x,5.12) -- (0.2*\x,5.0);
      }

      \foreach \y in {0,...,25} {
      \draw (0.0,0.2*\y) -- (5,0.2*\y);
      \draw[-] (-0.12,0.2*\y) -- (0.0, 0.2*\y);
      \draw[-] (5.0,0.2*\y) -- (5.12, 0.2*\y);
      }

      % Center 3x3 box for Automata
      \fill[white] (3.0,0.4) rectangle +(1.6,1.6);
      \draw[-,line width=3pt] (3.05,0.45) rectangle +(1.5,1.5);

      \fill[white] (3.0,0.85) rectangle +(0.4,0.1);
      \fill[white] (3.0,1.45) rectangle +(0.4,0.1);

      \draw[->, myblue] (2.7,1.5) -- (3.3,1.5);
      \draw[->, myblue] (3.3,0.9) -- (2.7,0.9);
      \node at (3.8,1.2) {$\mathcal{A}$};
    \end{scope}
    
      \begin{scope}[xshift=12cm]
      \foreach \x in {0,...,5}
      \draw (\x,-0.3) -- (\x,5.3);
      % Draw horizontal lines
      \foreach \y in {0,...,5}
      \draw (-0.3,\y) -- (5.3,\y);

      \foreach \x in {0,...,4} \foreach \y in {0,...,4}
      \draw (\x,\y) grid +(1,1);

      % Center 3x3 box for agent
      \draw[fill=white, line width=0pt] (0,0) rectangle (5,5);
      \fill[pattern=grid, pattern color=gray!50] (3.5,1.25) rectangle +(1.5,2.5);
      \draw[dashed, Triangle-Triangle] (0.15, 0.15) -- (4.85, 4.85);
      \draw[fill=gray!20] (2,2) rectangle +(1,1);
      \draw[line width=6pt] (0,0) rectangle (5,5);
      \draw[fill=white, line width=6pt] (1.5,1.5) rectangle +(2,2);

      \draw[white, dotted, line width=6pt] (2,3.5) -- (3,3.5);
      \draw[white, dotted, line width=6pt] (2,1.5) -- (3,1.5);

      \node[above,align=center] at (2.52,4) {$\pi$};
      \fill[white] (2.42,3.9) rectangle +(0.2, 0.2);
      \draw[-] (2.42,3.9) rectangle +(0.2, 0.2);

      \node[below,align=center] at (4,1.0) {$u$};
      \fill[white] (3.9,0.9) rectangle +(0.2, 0.2);
      \draw[-] (3.9,0.9) rectangle +(0.2, 0.2);

      \draw[<-, line width=1pt] (2.42,4) -- (1.0,4) -- (1.0,1) -- (2.52,1)  -- (2.52,2);
      \filldraw[fill=black] (2.52,1.0) circle (0.05);

      \draw[->, line width=1pt] (2.52,1)  -- (3.9,1);
      \draw[->, line width=1pt] (4,1.1)  -- (4,2.3);
      \draw[->, line width=1pt] (4, 2.8) -- (4, 4) -- (2.62,4);
      \draw[->, line width=1pt] (2.52, 3.9) -- (2.52, 3);
      \draw[->, line width=1pt] (4.5,2.5)  -- (4.7,2.5) -- (4.7, 1) -- (4.1, 1);

      \node[align=center] at (2.5,2.5) {$\Omega|_{b(\Theta)}$};
      \node at (2.25,4.6) {Self-Predicting $\mathcal{A}$gent};
      \node at (4.1,2.5) {$\Omega|_\Theta$};
    \end{scope}
  \end{tikzpicture}

%%% Local Variables:
%%% TeX-master: "../main.tex"
%%% TeX-command-extra-options: "-shell-escape"
%%% jinx-local-words: "Env center contin xshift yshift"
%%% End:
  \caption{
    \textbf{An illustrative depiction of a computationally-embedded agent interacting with its environment.} An embedded automaton,  $\mathcal{A} = (\Omega|_X, \Omega|_Y, \Omega|_\Theta, u, \pi)$, represents an agent embedded within the universal-local environment.
    \textbf{Left:} The agent is simulated by a boundaried Markov process within the environment, in which it iteratively receives an input from the environment, $x \in \Omega|_X$, produces the corresponding output $y = \pi(x; \theta) \in \Omega|_Y$, and updates its internal state, $\theta' = u(x; \theta) \in \Omega|_\Theta$. 
    \textbf{Right:} We will also consider an idealized setting in which a self-predicting agent exerts full control over its experience by reading and writing to an internal boundary-space, $\Omega|_{b(\Theta)}$.
    \looseness=-1
  }
  \vspace{-4mm}
  \label{fig:agents}
\end{figure}
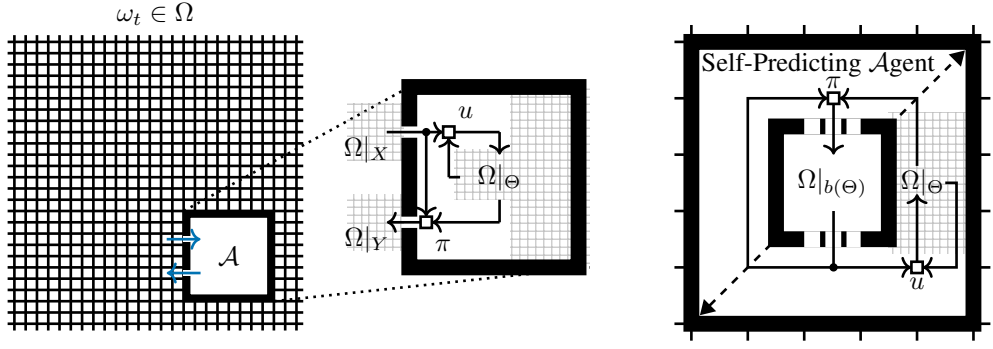

\subsection{Interactivity as a Computational Measure of Adapatability}

An agent's capability for learning can be characterized by its ability to adapt its future behaviour using its past experience.
We propose \emph{interactivity} to measure this capability directly in terms of algorithmic complexity.
Specifically, interactivity measures the difference between the algorithmic complexity of future behaviour with and without conditioning on past experience.

Following Proposition \ref{prop:agent}, we represent an embedded agent as an embedded automaton $\mathcal{A}$ where its input space constitutes its boundary-space, $\Omega|_{X} = \Omega|_{b^k(\Theta)}$.
Thus, the behaviour of the agent is determined by the values taken on the boundary-space, $b_{t}  \in \Omega|_{b^k(\Theta)}$ which combines the input and the corresponding output, $b_t = (x_t, y_t)$.
At time $t$, the behaviour can be separated into finite sequences: the past, $b_{0:t-1} := b_{0} b_{1}\cdots b_{t-1}$, and the $T$-horizon future, $b_{t:t+T-1} := b_{t} b_{t+1} \cdots b_{t+T-1}$.\looseness=-1
\begin{defn}[Interactivity]
Given an embedded automaton, ${\mathcal{A} := (\Omega|_X, \Omega|_Y, \Omega|_\Theta, u, \pi)}$, we define its $T$-horizon interactivity relative to its environment, $\mathcal{E}$, as the difference between the algorithmic complexity of its future behaviour with and without conditioning on its past behaviour,
  $$\mathbb{I}_{T}(\mathcal{A}\, |\, x_t, b_{0:t-1}) := \mathbb{K}_{\mathcal{E}}(b_{t:t+T-1} \, | \, \epsilon) - \mathbb{K}_{\mathcal{E}}(b_{t:t+T-1} \, | \, b_{0:t-1}),$$
  where $\epsilon$ is the empty string, $x_t$ is the current observation, $\theta_{t-1}$ is the internal state, $y_t = \pi(x_{t}; \theta_{t-1})$ is the action, $\theta_t = u(x_{t}; \theta_{t-1})$ is the next internal state, and $b_{t} = (x_{t}, y_{t})$ is the next behaviour tuple.
\end{defn}
That is, interactivity measures the predictable complexity of an agent's future behaviour, given its past behaviour.
Interactivity is high if both (i) the future behaviour, $b_{t:t+T-1}$, has high unconditional algorithmic complexity, and (ii) the past behaviour, $b_{0:t-1}$, is predictive of this future behaviour, thereby yielding a low conditional algorithmic complexity.
Interactivity-seeking behaviour thus balances complexity and predictability with respect to the agent's past behaviour, similar to how definitions of open-endedness balance novelty and learnability with respect to an observer \citep{hughes24_posit}.
 \looseness=-1

% \vspace{-3mm}
\subsection{An Interactivity-Seeking Agent Faces a Big World}
% \vspace{-1mm}
\label{sec:thesis}

The interactivity of any embedded agent is always constrained by its capacity.
That is, an interactivity-seeking agent can only sustain a given level of interactivity with a given capacity.
However, an interactivity-seeking agent can always use additional capacity to achieve higher interactivity.
\begin{theorem}
  \label{thm:big_world}
Given an embedded agent, $\mathcal{A} = (\Omega|_X, \Omega|_Y, \Omega|_\Theta, u, \pi)$, its maximum interactivity is asymptotically upper and lower bounded by a quantity that depends on its capacity.
\end{theorem}

The goal of an interactivity-seeking agent can be understood as balancing the complexity and predictability of its future behaviour, given its past behaviour. 
Seeking to achieve this goal requires that the agent continually adapt to its experience, learning about its computational limitations, and tracking its environment with predictions learned within these limitations.
This suggests the following interactivity thesis:
\begin{center}
  \emph{
    Interactivity measures a capability for continually adaptive behaviour.
  }
\end{center}
We refer to this as the interactivity thesis, rather than a hypothesis, to reflect its speculative and philosophical nature.
With low interactivity, the future behaviour of an agent is either
i) simple, or ii) complex and unpredictable.
In either case, the thesis asserts that the agent's capability for continually adaptive behaviour is limited.
A simple agent has a limited range of possible behaviours and thus has a relatively lower capability for adaptive behaviour.
A complex agent could have a relatively high capability for adaptive behaviour, but only if its future behaviour is influenced by, and can be predicted from, past experience.
Embracing the interactivity thesis naturally leads to a spectrum of adaptive capability.
\looseness=-1

%%% Local Variables:
%%% TeX-master: "../main.tex"
%%% TeX-command-extra-options: "-shell-escape"
%%% End:

\section{Maximizing Agent-Relative Interactivity with Reinforcement Learning}
\label{sec:approx}
Maximizing interactivity poses a fundamental computational challenge: it depends on algorithmic complexity, which is generally uncomputable.
While algorithmic complexity can be computed for finite automata through exhaustive program enumeration \citep{li19_introd_kolmog_compl_its_applic_edition}, this requires computational resources that exceed the automaton's own capacity.
Consequently, an interactivity-seeking agent must approximate interactivity rather than compute it exactly. 

To approximate interactivity, we take a distortion-rate view of algorithmic complexity, which measures complexity in terms of prediction error under a constrained reference machine, rather than under an unconstrained universal Turing machine \citep{vereshchagin10_rate_kolmog}.
The embedded agent, being an automaton, provides a natural choice for the constrained reference machine.
We then augment this agent to include a predictor function as a function of agent-state, and measure agent-relative complexity by the incurred temporal difference errors under this predictor.
\begin{defn}
  Given an embedded agent $\mathcal{A} = (\Omega|_X, \Omega|_Y, \Omega|_\Theta, u, \pi)$ with a predictor ${v : \Omega|_{X\cup Y} \times \Omega|_\Theta  \to \Omega|_{X\cup Y}}$ and $\gamma \in [0,1]$, the agent-relative complexity of future behaviour $b_{t:t+T-1}$, conditioned on past behaviour $b_{0:t-1}$, is the sum of future temporal difference errors, $$\hat{\mathbb{K}}_{\mathcal{A}}(b_{t:t+T-1}|b_{0:t-1}) := \sum_{k=0}^{T-1}\left(b_{t+k} + \gamma v(b_{t+k}; \theta_{t+k-1}) - v(b_{t+k-1}; \theta_{t+k-1})\right)^2,$$
where $\theta_{t+k-1}$ is the agent-state after processing $b_{0:t-1}$ followed by $b_{t:t+k-1}$ through repeated application of the agent's update function $u$.
\end{defn}

Complexity, measured in this way, is relative to the agent's capabilities.
If the future prediction errors are large (small), then the future behaviour is relatively complex (simple).
These future prediction errors depend crucially on the agent through its predictor, its current agent-state, its learning algorithm, its policy, and its observations from the environment.
With this prediction error formulation of complexity, we now consider the agent-relative interactivity,
$$\hat{\mathbb{I}}_{T}(\mathcal{A}\, |\, x_t, b_{0:t-1}) := \hat{\mathbb{K}}_{\mathcal{A}}(b_{t:t+T-1} \, |\, \epsilon) - \hat{\mathbb{K}}_{\mathcal{A}}(b_{t:t+T-1} | b_{0:t-1}).$$

\subsection{Learning to Maximize Agent-Relative Interactivity}

We now develop a reinforcement learning algorithm for maximizing agent-relative interactivity.
Our algorithm is summarized in the following three steps:
(i) learning a prediction of the discounted future behaviour using a value function,
(ii) computing the agent-relative interactivity, defined as the difference between static prediction errors (unconditional complexity) and dynamic prediction errors (conditional complexity), and
(iii) meta-learning a policy to maximize agent-relative interactivity.

The first step involves learning a prediction of the future input-output behaviour, where we consider the deterministic setting and omit the expectation over trajectories.
This value function predicts both future observations and actions, generalizing
successor features \citep{barreto17_succes} and the successor representation \citep{dayan93_improv}.
Specifically, with temporal difference learning \citep{sutton84_tempor,sutton88_learn}, we train a value function to predict the discounted sum of future input-output behaviour, $$v(b_t; \theta_t) \approx \sum_{k=0}^\infty \gamma^{k}b_{t+k+1},\quad \delta_{t+k}(\theta) = b_{t+k} + \gamma v(b_{t+k}; \theta) - v(b_{t+k-1}; \theta),$$
where $\theta_{t+k}$ can be understood as dynamic parameters which are updated using semi-gradient TD($0$): $\theta_{t+k} = \theta_{t+k-1} + \eta\delta_{t+k}(\theta_{t+k-1})\nabla_\theta v(b_{t+k-1}; \theta)\big|_{\theta=\theta_{t+k-1}}$, where $\eta$ is a step-size. 

For the second step, we measure the conditional and unconditional complexity terms using dynamic and static temporal difference errors, respectively.
From the first step, we can readily compute the conditional complexity using the dynamic temporal difference errors just described,
$$\hat{\mathbb{K}}_{\mathcal{A}}(b_{t+1:t+T} | b_{0:t-1}) = \sum_{k=0}^{T-1}\delta^2_{t+k}(\theta_{t+k-1}).$$
Unlike the conditional complexity term, the unconditional complexity term is not conditioned on past experience, which means means that it cannot be computed using the dynamic parameters.
Instead, computing the unconditional complexity involves a static and unchanging reference state, $\theta_{ref}$.
One convenient choice is the current state, $\theta_{ref} := \theta_{t-1}$, which we adopt to yield the following form for agent-relative interactivity,
$$\hat{\mathbb{I}}_{T}(\mathcal{A}\, |\, x_t, b_{0:t-1}) \approx \sum_{k=1}^{T-1} \left(\underbrace{\delta_{t+k}^{2}(\theta_{t-1})}_{\text{static}} -  \underbrace{\delta_{t+k}^{2}(\theta_{t+k-1})}_{\text{dynamic}}\right).$$

Lastly, we outline how a policy can be trained to maximize interactivity.
To obtain the future prediction errors, we assume access to a differentiable model that we can use to roll-out a sequence of observations and actions using the current policy.
Interactivity is estimated on the roll-out by computing the cumulative difference between the static and dynamic prediction errors.
The policy is then updated using a gradient-based optimizer with the objective of maximizing the estimated interactivity.\footnote{This optimization problem involves meta-gradients due to the dynamic prediction errors that depend on the value function's parameter update.}
Crucially, both the policy and value function must be continually updated to sustain interactivity: if the value function stops changing then interactivity is trivially zero, and if the policy stops changing then the value function can converge to a fixed point.

\vspace{-3mm}
\subsection{Maximizing Interactivity as a Continual Learning Problem}
\vspace{-1mm}
We now show that an interactivity-seeking reinforcement learning agent faces a continual learning problem.
Intuitively, this is because interactivity-seeking behaviour is directed towards experience which can be predicted by a dynamic value function that learns from experience, but which cannot be predicted by the current value function if it were to remain static.
Specifically, interactivity-seeking agents are suboptimal if they stop learning, such as when the parameters of their policy and value stop changing.
Similar to our previous result on embedded automata (Theorem \ref{thm:big_world}), we also show that the agent's maximum possible interactivity increases as the its capacity increases.

\begin{theorem}[Big World]
  \label{thm:continual_learning}
  An agent that seeks to maximize its agent-relative interactivity is (i) limited by its finite capacity and, (ii) suboptimal if it stops learning.
\end{theorem}
The desiderata of Theorem \ref{thm:continual_learning} were previously described as conditions for a big world simulator \citep{kumar24_need_big_world_simul}.
Thus, interactivity-seeking appears to be a general problem setting which captures the big world hypothesis: the best use of an agent's limited capacity is to continually adapt.

%%% Local Variables:
%%% TeX-master: "../main.tex"
%%% TeX-command-extra-options: "-shell-escape"
%%% End:

\vspace{-2mm}
\section{Evaluating Continual Adaptation With Agent-Relative Interactivity}
\vspace{-2mm}
\label{sec:behaviour}
\begin{figure}[t]
  \centering
\includegraphics[width=0.32\linewidth]{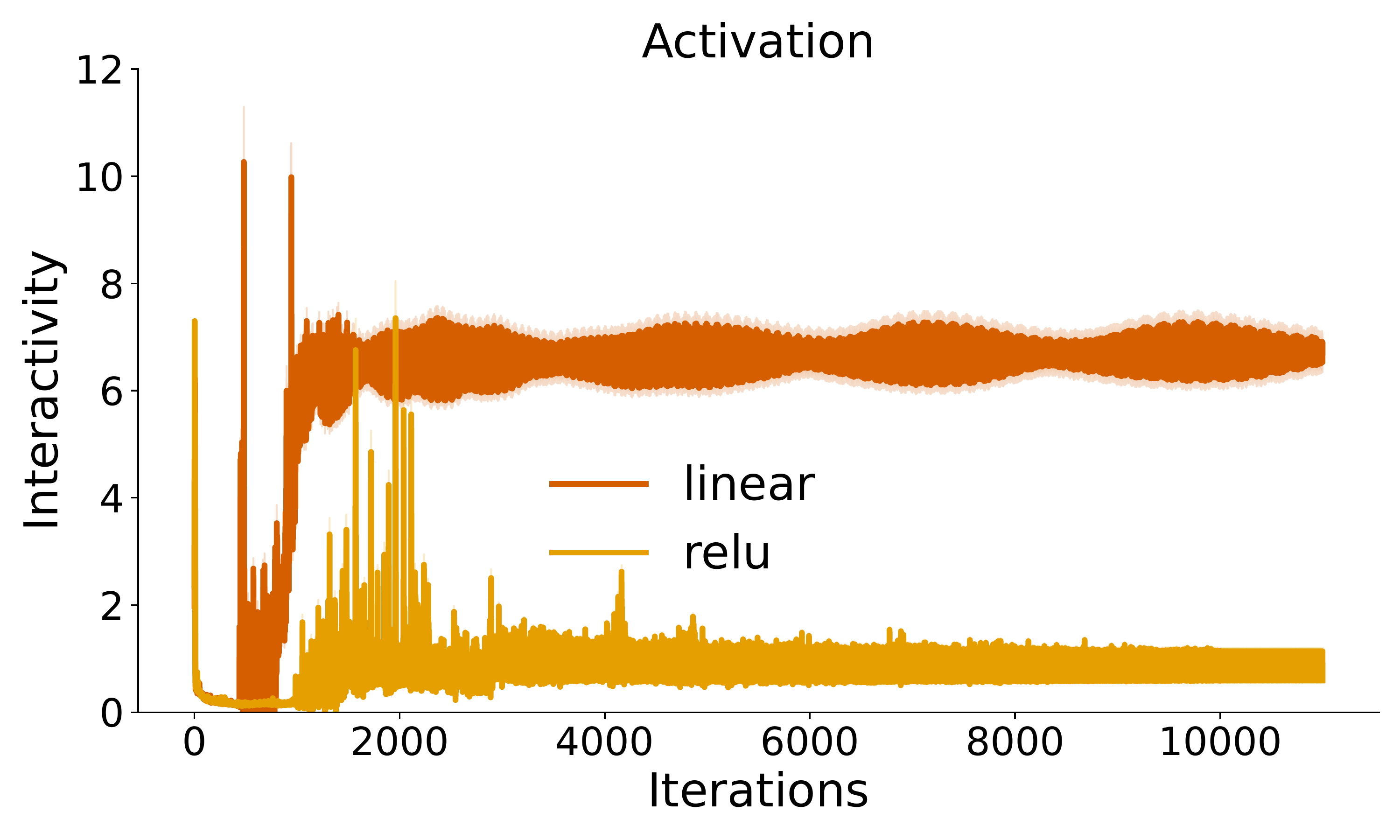}
\includegraphics[width=0.32\linewidth]{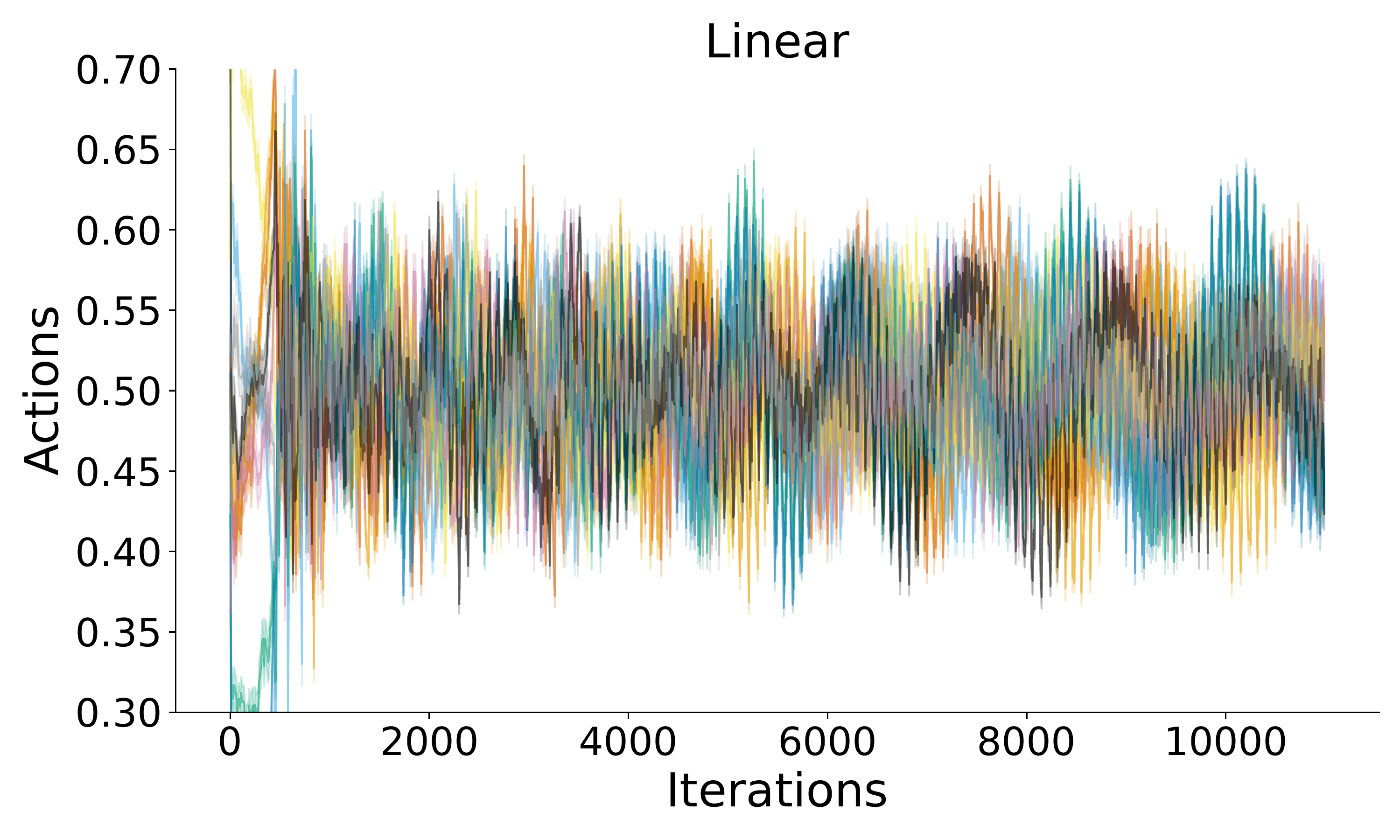}
\includegraphics[width=0.32\linewidth]{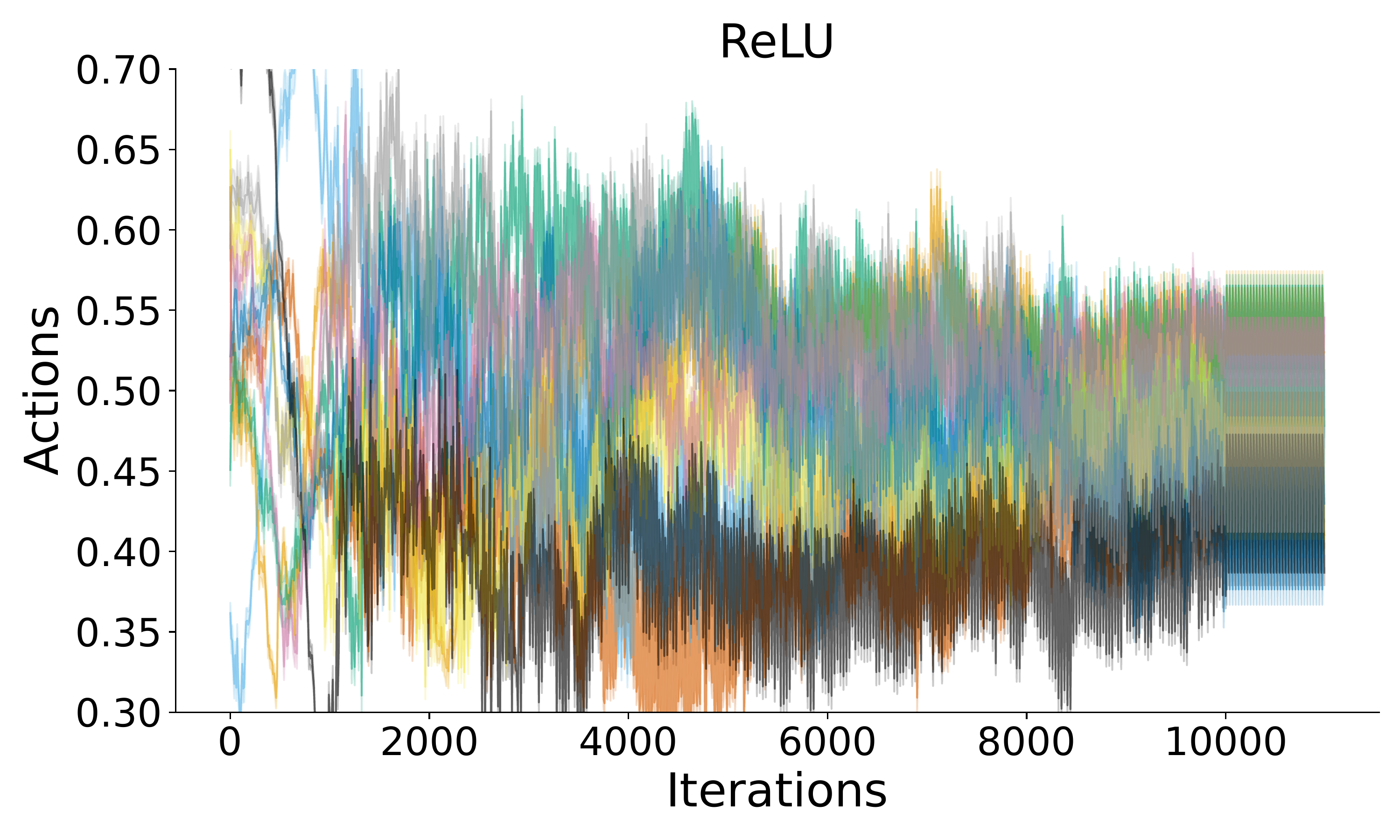}
\caption{
    \textbf{Deep \texttt{ReLU} networks fail to sustain performance on the interactivity evaluation task.}
    Interactivity-seeking simulates a big world and directly evaluates a learning algorithm's capability for continually adaptive behaviour. \textbf{Right:} The deep linear network sustains interactivity, whereas the deep \texttt{ReLU} collapses in performance.
    \textbf{Middle:} Each colour corresponds to one component of the action vector.  The deep linear policy learns to produce actions following a non-stationary wave, which can be locally predicted by a static linear function and globally predicted by a dynamic linear function. \textbf{Left:} The deep \texttt{ReLU} policy fails to produce actions with any predictable structure, thus highlighting the challenge of interactivity-seeking as a continual learning problem.
}
  \label{fig:behavioural}
\vspace{-4mm}
\end{figure}

We now use interactivity-seeking to construct a synthetic problem that evaluate a learning algorithm's capability for continually adaptive behaviour. 
In particular, we consider a setting in which the agent only observes its own actions (see the self-predicting agent in Figure \ref{fig:agents}, right).
Even though such an agent has full control over its experience stream, its interactivity is still implicitly constrained by its capacity.
That is, the interactivity objective depends on the parameters and learning algorithm of the value function, which the agent does not directly observe.
An advantage of this evaluation approach is that it is environment-free, allowing direct evaluation of an algorithm outside of an environment, or any collected data.
Instead, algorithms are directly evaluated by continually learning from their own online experience in a manner similar to self-play.

We instantiate the reinforcement learning agent outlined in Section \ref{sec:approx} with a linear parameterization of the value function, ${v(b_t; \mathbf{W}_t) := \mathbf{W}_t b_t \approx \sum_{k=0}^\infty \gamma^{k}b_{t+k+1}}$.
Linearity provides stability for learning online with TD($0$), compared to temporal difference learning with deep nonlinear networks.
For the policy parameterization, we consider a deep network architecture (using either linear or \texttt{ReLU} activations), where we normalize the output to ensure that the output has bounded range,
${b_{t+1} := \text{RMSNorm}\left(\pi(b_t; \theta_t)\right)}$.
This policy is optimized using the model-based approach described in Section \ref{sec:approx}.
Optimizing the agent-relative interactivity is a bi-level optimization problem as the dynamic prediction errors depend on the value function's learning process.
Thus, our model-based approach to maximizing interactivity is similar to
the cross-prop algorithm \citep{veeriah17_cross}, which is an online version of 
model-agnostic meta-learning \citep{finn17_model}.
For both the policy and the value function, we found RMSProp \citep{hinton12_neural_networ_machin_learn} to balance performance and stability better than either Adam \citep{kingma14_adam}, or vanilla gradient descent.

Our results demonstrate that a deep nonlinear policy is unable to sustain interactivity (see Figure \ref{fig:behavioural}, right).
That is, the deep nonlinear policy is unable to plan an action sequence for which
the dynamic value function has low prediction error,
but for which the current static value function has high prediction error.
However, we find that a deep linear policy is able to sustain interactivity.
This finding suggests that interactivity-seeking agents produce non-stationarity that can lead to apparent loss of plasticity, which linear methods have been shown to avoid \citep{lewandowski25_plast_learn_deep_fourier_featur,dohare24_loss}.  
Observing the actions chosen by each policy, we found that the deep linear policy learned to produce actions with predictable structure, resembling a non-stationary wave (see Figure \ref{fig:behavioural}, middle).
These actions can be locally predicted by a linear function, but global prediction requires a dynamic linear function.
In contrast, the deep nonlinear policy learned failed to produce actions with any predictable structure (see Figure \ref{fig:behavioural}, left).
Furthermore, in Figure \ref{fig:depth_width}, we found that deep linear networks are also capable of increasing their interactivity with more capacity, in the form of deeper or wider networks.
These findings demonstrate that interactivity-seeking simulates the challenges of a big world in which
any policy is limited by its finite capacity, and suboptimal if it stops learning.
\begin{figure}
  \centering
  \vspace{-4mm}
\includegraphics[width=0.49\linewidth]{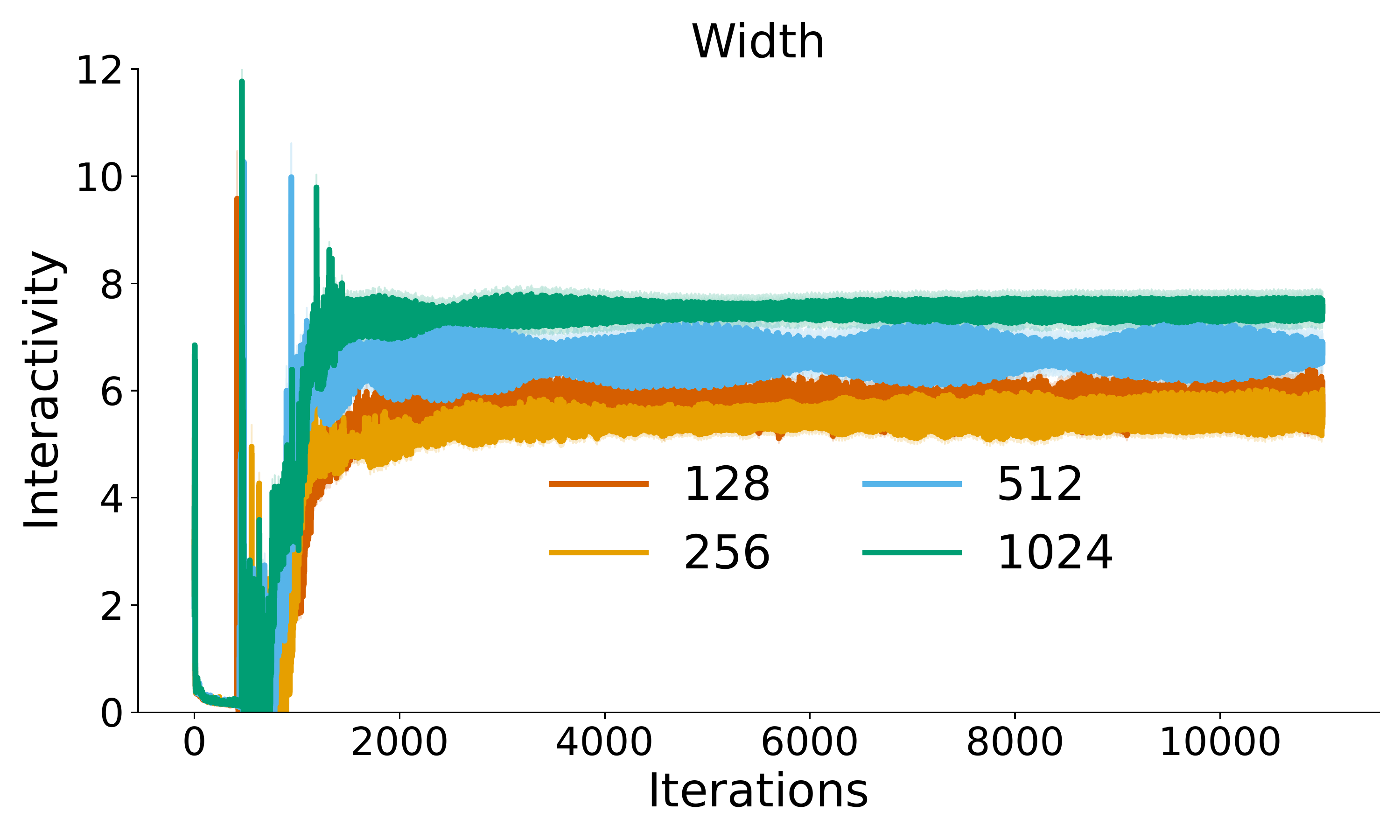}
\includegraphics[width=0.49\linewidth]{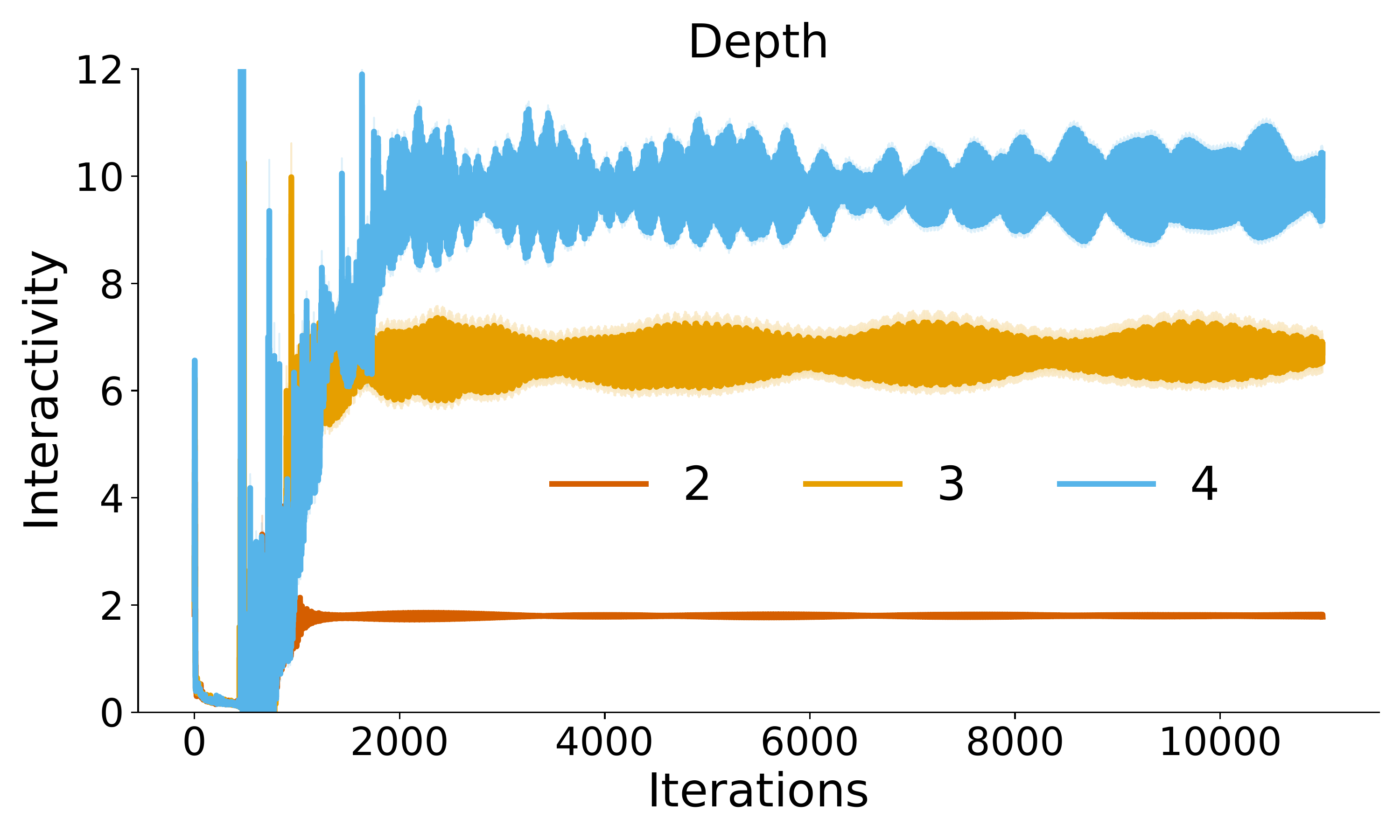}\\
\caption{
  \textbf{Deep linear networks are capable of sustaining higher interactivity with more computational resources.} By increasing the width and depth of the deep linear network, we increase the network's capacity for continual adaptation, which allows it more quickly change its linear function approximator. \textbf{Left:} Increasing width marginally increases the sustained interactivity. \textbf{Right:} Increasing depth results in a large increase in interactivity, as well as more oscillatory behaviour.
  }
  \label{fig:depth_width}
  % \vspace{-10mm}
\end{figure}

%%% Local Variables:
%%% TeX-master: "../main.tex"
%%% TeX-command-extra-options: "-shell-escape"
%%% End:

\vspace{-3mm}
\section{Discussion}
\vspace{-2mm}

In this paper, we introduced a computationally-embedded perspective on the big world hypothesis, where we 
(i) characterize the implicit constraint faced by an embedded agent,
(ii) propose interactivity as a computational measure of adaptability, and
(iii) develop a reinforcement learning algorithm for maximizing interactivity.
We show that interactivity-seeking leads to the common desideratum of the continual learning problem in which any agent that stops learning is suboptimal.

Our work departs from dogmas common in reinforcement learning \citep{abel24_three_dogmas_reinf_learn} by formalizing the agent as embedded in the environment, and by proposing interactivity which is defined relative to an agent.
Interactivity-seeking agents are implicitly constrained because the interactivity objective depends on components, such as parameters, that are not directly observable but evolve through learning.
As we demonstrated in the environment-free evaluation task, this perspective has a critical consequence: an interactivity-seeking agent must continually adapt rather than converging to a fixed point.
Unlike continual learning problems which often involve manually designed non-stationarity, an interactivity-seeking agent autonomously produces the non-stationarity through its learning algorithm; the complexity of the task scales with the complexity of the learning algorithm.
This evaluation task thus provides a unique challenge, incentivizing the design of continual learning algorithms that achieve stable and fast adaptation, such as through alternative sequence architectures, learning algorithms designed for continual learning \citep{lewandowski25_learn_contin_spect_regul,dohare24_loss}, or better procedures for selecting hyperparameters \citep{mesbahi25_posit}.

Beyond serving as an evaluation task for continual learning, interactivity can also serve as an auxiliary objective for agents in conventional reinforcement learning environments.
Specifically, as an intrinsic reward, interactivity would incentivize directed exploration similar to curiosity-driven methods.
However, estimating interactivity in an external environment would involve learning a model of the environment to calculate meta-gradients, or a value function that directly approximates interactivity.
\looseness=-1

This practical potential for directed exploration suggests a refinement of the interactivity thesis. We speculate that if an agent can sustain a particular level of interactivity, then it possesses the capacity to learn any behaviour with equal or lower interactivity---including those that maximize the cumulative sum of expected discounted reward.
That is, learning algorithms that sustain high interactivity indicate a capacity needed to learn other goal-directed behaviours.

\vspace{-3mm}

%%% Local Variables:
%%% TeX-master: "../main.tex"
%%% TeX-command-extra-options: "-shell-escape"
%%% End:

%%% Local Variables:
%%% TeX-master: "neurips.tex"
%%% TeX-command-extra-options: "-shell-escape"
%%% End:

\subsubsection*{Acknowledgments}
We would like to thank Saurabh Kumar and Hong Jun Jeon for helpful discussion during earlier stages of this work, David Abel and Anna Harutyunyan for their encouraging discussion regarding dogmas in reinforcement learning, and the anonymous reviewers for their feedback.
The research is supported in part by the Natural Sciences and Engineering Research Council of Canada (NSERC), the Canada CIFAR AI Chair Program, and the Digital Research Alliance of Canada.
\looseness=-1

\bibliography{references}
\bibliographystyle{apalike}

\newpage
\appendix

\section{Proofs}
\label{sec:proofs}

%%%
%%%
%%%

\begin{proof}[Proof of Proposition \ref{prop:markov_env}]
Let $M = (Q, \Sigma', \Gamma, \delta, q_0, \square', F)$ be a Turing machine where:
\begin{itemize}
  \item  $Q$ is the finite set of states
  \item $\Sigma'$ is the input alphabet 
  \item $\Gamma$ is the tape alphabet with $\Sigma' \subset \Gamma$
  \item  $\delta: Q \times \Gamma \to Q \times \Gamma \times \{L,R\}$ is the transition function
  \item  $q_0 \in Q$ is the initial state
  \item  $\square' \in \Gamma$ is the blank symbol
  \item  $F \subseteq Q$ is the set of final states
\end{itemize}

Our proof constructs an algorithmic Markov process, $\mathcal{E} = (\Omega, \Xi, \mathbb{T})$.
Specifically, it involves showing how the Turing machine can be represented on a state-space $\Omega$, with its transition function represented in the Markov transition function, $\mathbb{T}$.

 \textbf{Constructing the state-space ($\Omega$)}

 Let $\Xi = \mathbb{Z}$ denote the integers for tape positions, $\Sigma = ((Q \cup \{\square''\}) \times \Gamma)  \cup \{\square\}$ where $\square$ is the Markov blank symbol.
 Then each $\omega \in \Omega$ represents a complete Turing machine configuration by encoding the tape contents at each position, the current state and head position.

Specifically, for a TM configuration with state $q$, head at position $h$, and tape contents $...a_{-1}a_0a_1...$, we define:
$$\omega(\xi) = \begin{cases}
(q, a_h) & \text{if } \xi = h \text{ (head position)} \\
(\square'', a_\xi) & \text{if } \xi \neq h \text{ and } a_\xi \neq \square' \\
\square & \text{otherwise}
\end{cases}$$

 \textbf{Constructing the transition function ($\mathbb{T}$)}

The Markov transition $\mathbb{T}(\omega)$ simulates one step of the Turing machine. It first finds the head position $h$ where $\omega(h) = (q,a)$ for some $q \in Q, a \in \Gamma$. It then
applies the Turing machine transition: $\delta(q,a) = (q', a', d)$ where $d \in \{L,R\}$. Lastly, 
we construct $\omega'$ by:
\begin{itemize}
\item Setting $\omega'(h) = (\square'', a')$ (write new symbol)
  \item Setting $\omega'(h + \text{offset}(d)) = (q', b)$ where

        $$b = \begin{cases}
c & \text{if } \omega(h + \text{offset}(d)) = (\square'', c) \text{ for some } c \in \Gamma \\
\square' & \text{if } \omega(h + \text{offset}(d)) = \square
        \end{cases}$$
        and offset$(L) = -1$, offset$(R) = 1$
\item Leaving all other positions unchanged: $\omega'(\xi) = \omega(\xi)$ for all $\xi \notin \{h, h + \text{offset}(d)\}$
\end{itemize}

At initialization, we have that the initial tape contents is finite and thus $|\omega'| < \infty$.
At each step, the state remains finite, $|\omega'| < \infty$, because each step changes at most 2 positions.
Note also that $\mathbb{T}$ is computable in $O(\text{poly}(|\omega|))$; finding the head and updating positions requires linear scan and constant updates.
\end{proof}

%%%
%%%
%%%

\begin{proof}[Proof of Proposition \ref{prop:agent}]
We prove each part in sequence.

\textbf{Part 1:} The automaton is equivalent to a $k$-horizon boundaried Markov process.

Given that $b^k(\Theta) = X$, the $k$-step transition dynamics on the internal state space $\Omega|_\Theta$ depends only on the current internal state and the current observation.

We define two $k$-horizon boundaried Markov process.
First, for the internal statel, we have: $\mathcal{E}|_\Theta = (\Omega|_\Theta, \Omega|_{X}, \mathbb{T}|_\Theta^k)$ where $\mathbb{T}|_\Theta^k(\theta, x)) := u(x; \theta)$.
Next, for the policy, we have: $\mathcal{E}|_Y = (\Omega|_Y, \Omega|_{X \cup \Theta}, \mathbb{T}|_Y^k)$ where $\mathbb{T}|_Y^k(y, \theta, x) := \pi(x; \theta)$.

That is, both the policy, $\pi$, and the internal state update function, $u$, is simulated by the $k$-step composition of the environment's transition function.
By the definition of uniform locality and the condition $b^k(\Theta) = X $, the $k$-step transition on $\Omega|_\Theta$ is fully determined by the current internal state $\theta \in \Omega|_\Theta$ and the boundary $x \in \Omega|_{X}$.
This establishes the equivalence.

\textbf{Part 2:} The automaton's environment is a partially observable Markov decision process.

For the automaton specifically, the environment consists of:
\begin{itemize}
\item \textbf{State space:} $\Omega$
\item \textbf{Action space:} The output space $\Omega|_Y$ 
\item \textbf{Observation space:} The input space $\Omega|_X$
\item \textbf{Transition function:} Given the current environment state, the automaton's action $y \in \Omega|_Y$, the next environment state follows from $\mathbb{T}$
\item \textbf{Observation function:} The automaton observes $x \in \Omega|_X$ from the current environment state
\end{itemize}

Since the automaton only observes $\Omega|_X$ and not the full environment state $\Omega$, this constitutes partial observability. The Markov property holds for the underlying environment state transitions via $\mathbb{T}$.

\textbf{Part 3:} The automaton's interaction is equivalent to a stateful policy acting on the environment.

The automaton maintains an internal state $\theta \in \Omega|_\Theta$ and produces actions via the output function $\pi: \Omega|_X \times \Omega|_\Theta  \to \Omega|_Y$. This defines a stateful policy,
$\pi(x; \theta)$, where the internal state $\theta$ is updated according to:
$$\theta_{t+1} = u(x_t; \theta_{t-1}).$$

This is precisely the definition of a stateful policy in partially observable environments, where the policy maintains internal memory (the substate $\theta$) and conditions its actions on both observations and this internal state.

\end{proof}

%%%
%%%
%%%

\begin{proof}[Proof of Proposition \ref{prop:implicit}]

  An embedded automaton has a finite number of states, which is upper bounded by the size of its internal state space, $|\Theta|$.
  This upper bounds the number of unique configurations that the internal state space can take, which thus upper bounds the capacity of the automaton.

  Thus, the automaton is constrained to have finite capacity upper bounded by $|\Theta|$.

  We further show that an embedded automaton is only capable of a limited form of computation relative to the partially observable Markov decision process. We then outline additional constraints that result from the automaton being embedded.

  An embedded automaton is equivalent to a finite-state machine.
  This means that the automaton is only capable of recognizing a regular language.
  The partially observable Markov decision process that it faces, however, is a function of an unbounded substate-space of the computationally universal environment.
  This means that the environment can, in general, generate a recursively-enumerable language.
  The embedded automaton is thus implicitly computationally constrained, because
  of the separation between finite-state automata and the turing-complete environment in the Chomsky hierarchy \citep{chomsky59}.

  Moreover, for a sufficiently long behaviour sequence, the automaton must eventually return to a previous state (\textit{i.e.}, $T >> |\Theta|$).
  Any behavioural sequence with a period longer than $|\Theta|$ cannot be represented by the automaton.

  Thus, an embedded automaton simulated in a universal-local environment is implicitly constrained: there exist input-output behaviours that it cannot realize.

  There are two additional ways in which an embedded automaton is implicitly constrained:

  \begin{enumerate}
    \item \textbf{Minimum size}: The size of an embedded automaton, including the size of its input and output spaces, cannot be arbitrarily small, and thus there exists a minimum size. This implies that the automaton cannot read and write to arbitrarily small parts of the environment, constraining its observation and action spaces.

    \item \textbf{Simulation time}: Simulating an embedded automaton in a universal-local environment may also incur a simulation overhead.
          This constrains the automaton by the fact that several transitions in the environment may be necessary to simulate a single transition for the automaton.
  \end{enumerate}

  While the embedded automaton is computationally constrained relative to its environment, these two additional constraints limit the information made available to the automaton about the environment.
  Specifically,
  an automaton generally cannot observe, process and output information at the same granularity, or at the same timescale, as the environment because of constraints on its size and its simulation time.
\end{proof}

%%%
%%%
%%%

\begin{proof}[Proof of Theorem \ref{thm:big_world}]

  We denote the capacity of the automaton as $C(\mathcal{A})$, and denote the maximum achievable $T$-horizon interactivity for a given automaton as,
  $$\max_{\mathcal{A}}\mathbb{I}_{T}(\mathcal{A}|x_t, b_{0:t-1}).$$

  Intuition: interactivity is determined by the dependence between the past and future behaviour of the automaton. This behaviour is determinde by a substate-space that grows with the horizon ($T$).
  Specifically, the future behaviour of an embedded automaton is determined by (i) the universal-local environment's transition function, (ii) the embedded automaton's initial internal state, and (iii) a substate-space of the environment that grows with the horizon of behaviour considered.
  If an embedded automaton's past behaviour were predictive of its future behaviour of a given horizon, then it would also imply that its past behaviour is predictive of a substate-space growing with that horizon.
  For a large enough horizon, the size of this substate-space will eventually be larger than the capacity of the automaton.
  An embedded automaton with a given capacity cannot maximize interactivity beyond a given horizon, meaning that it actively faces an implicit capacity constraint.

  Formally: an automaton with sufficiently high interactivity will produce a behaviour with high unconditional Kolmogorov complexity, $\mathbb{K}(b_{t:t+T-1}) + O(1) > C(\mathcal{A})$, where $O(1)$ is a constant independent of the automaton. 
  Because the behaviour is generated by the automaton, we know that the Kolmogorov complexity is also upper bounded by the capacity, $C(\mathcal{A}) + O(1) > \mathbb{K}(b_{t:t+T-1})$.

  Thus, we can upper bounded interactivity in terms of the unconditional complexity,
  \begin{align*}
    \mathbb{I}_{T}(\mathcal{A}|x_t, b_{0:t-1}) &= \left(\mathbb{K}(b_{t:t+T-1}) - \mathbb{K}(b_{t:t+T-1} \, | \, b_{0:t-1} )\right)\\
    &\leq \mathbb{K}(b_{t:t+T-1})
  \end{align*}
  Which uses the fact that the conditional algorithmic complexity is positive, $\mathbb{K}(b_{t:t+T-1} \, | \, b_{0:t-1} ) > 0$.

  Next we use the fact that the Kolmogorov complexity of a sequence produced by an automaton is upper bounded by its capacity, which implies that,

  $$ \max_{\mathcal{A}}\mathbb{I}_{T}(\mathcal{A}|x_t, b_{0:t-1}) \leq \mathbb{K}(b_{t:t+T-1}) \leq C(\mathcal{A}) + O(1)$$

  \textbf{Lower Bounding Interactivity By Capacity}

  For the lower bound, we provide a constructive proof in which we show how an automaton can scale a lower bound on its interactivity proportional to its capacity.

  In particular, the automaton will use a brute-force method which simulates a much smaller automaton along with a stored history.
  That is, a subroutine will simulate an automaton with capacity $\alpha C(\mathcal{A})$, where $\alpha < 1$.
  The automaton will use its remaining capacity to store on the order of $(1-\alpha) C(\mathcal{A})$ of the previous input-output pairs.
  It then runs an enumeration strategy as a subroutine, which simulates the smaller automaton with capacity $\alpha C(\mathcal{A})$.
  Specifically, the subroutine enumerates over all possible $T$-horizon futures and selects the next output that maximizes the resulting $T$ horizon unconditional complexity and minimizes the $T$-horizon conditional complexity with respect to the stored history of length $(1-\alpha) C(\mathcal{A})$. While this is computationally expensive, we can choose $\alpha$ to be small enough that this is possible. The interactivity achieved in this setting is thus,$\alpha C(\mathcal{A}) - O(!)$

  Putting these together, we have the lower and upper bounds of:

  $$\alpha C(\mathcal{A}) - O(1) <  \max_{\mathcal{A}}\mathbb{I}_{T}(\mathcal{A}|x_t, b_{0:t-1}) \leq \mathbb{K}(b_{t:t+T-1}) \leq C(\mathcal{A}) + O(1)$$

\end{proof}

%%%
%%%
%%%

\begin{proof}[Proof of Theorem \ref{thm:continual_learning}]

  We provide a proof for each of the two desiderata

\emph{(i)}
The first property follows from an argument that is similar to Theorem \ref{thm:big_world}, but adapted to a learning agent, $\mathcal{A}$, with some bounded capacity $C(\mathcal{A})$. We are interested in what interactivity the best such agent can achieve, $\max_{\mathcal{A}}\mathbb{I}_{T}(\mathcal{A}|x_t, b_{0:t-1})$.

A bounded agent that maximizes its interactivity will have a non-zero unconditional agent-relativized complexity, $\mathbb{K}_{\mathcal{A}}(b_{t:t+T-1} | \epsilon) > 0$ (otherwise, its interactivity would be zero).
This implies that the unconditional Kolmogorov complexity of its behaviour is on the order of the the capacity of the agent, $\mathbb{K}(b_{t:H}) \geq C(\mathcal{A}) - O(1)$, where $O(1)$ is a constant independent of the agent.
Because the behaviour is generated by the automaton, we know that the Kolmogorov complexity is also upper bounded in terms of the capacity, $C(\mathcal{A}) + O(1) \geq \mathbb{K}(b_{t:H})$.

Such an agent will also have low conditional agent-relativized complexity (otherwise, its interactivity would be low).
An optimal learning agent that minimizes the agent-relativized complexity,
$\mathbb{K}_\mathcal{A}(b_{t:t+T-1}| b_{0:t-1}) = 0$, has conditional Kolmogorov complexity is strictly less than the capacity of the agent, $\mathbb{K}(b_{t:t+T-1} | b_{0:t-1}) < C(\mathcal{A})$.
  In fact, we have, for $\alpha < 1$, that
  $\mathbb{K}(b_{t:t+T-1} | b_{0:t-1}) \leq \alpha C(\mathcal{A})$.
  This is because the agent can only use a fraction of its capacity on predicting its future behaviour (in addition to making predictions, an agent selects actions, and updates its substate).

Taken together, we have that the performance of an interactivity-seeking agent interactivity is upper and lower bounded by capacity,
$$(1-\alpha)C(\mathcal{A}) - O(1) \leq \max_{\mathcal{A}}\mathbb{I}_{T}(\mathcal{A}|x_t, b_{0:t-1}) \leq C(\mathcal{A}) + O(1).$$
An agent with a given capacity cannot maximize its interactivity without increasing its capacity.
Thus, a bounded agent that seeks to maximize its interactivity through learning is limited by its finite capacity constraint.

\emph{(ii)}
For the second property, we demonstrate the necessity of continual adaptation for maximizing interactivity, by considering the role of the embedded agent's transition function.

\textbf{Continual Adaptation in Automata}

First we consider the
finite-state automaton, $\mathcal{A}$, and how its substate transition function, $u$, encodes its learning.
An automaton agent that has stopped learning is thus equivalent to one that stops updating its internal state.
In this case, the automaton's internal state remains constant $z_{t'} = z$ for all $t' > t$.
A finite-state automaton has a capacity on the order of $C(\mathcal{A}) \in O(poly(|A|))$.
But, a finite-state automaton that does not update its internal state, denoted by $\mathcal{A}^{-}$, has a reduction in its capacity.
In particular, the capacity is reduced to $C(\mathcal{A}^{-}) = O(|\Omega|_{X}|)$, because the terms needed to encode the transition function are no longer needed for an automaton that does not use the transition function.
Using the upper bounds on interactivity from the Theorem 1, we conclude that an agent that stops learning reduces its future output complexity from
$O(|I_{A}||A| \log |A|)$
to
$O(|I_{A}|)$.
Thus, it is suboptimal to stop learning.

\textbf{Continual Adaptation in Reinforcement Learning Agents}

\textbf{Value parameters}: If the parameters of the value function stop being updated, the interactivity objective immediately collapses to $0$.

\textbf{Policy parameters}: Suppose the policy's parameters stop being updated, meaning that the policy becomes fixed.
Then the sequence of actions taken by the fixed policy becomes a Markov process, which is predictable. 
Under the Markovian dynamics induced by a fixed policy, the value function could converge to an optimal static prediction of the future behaviour. This would lower both the unconditional Kolmogorov complexity and interactivity.
Thus, an agent maximizing interactivity is suboptimal if it stops updating its policy parameters.

\end{proof}

%%% Local Variables:
%%% TeX-master: "../main.tex"
%%% TeX-command-extra-options: "-shell-escape"
%%% End:

\section{Experimental Details for Behavioural Self-Prediction}
\label{sec:exps}

The code for the experiments can be found at: \href{https://github.com/AlexLewandowski/bigger-world-interactivity}{https://github.com/AlexLewandowski/bigger-world-interactivity}

The problem that we consider involves a value function predicting a policy's actions, where the policy is trained to maximize interactivity.
Specifically, the policy is adversarial to the static value function but cooperative with the dynamic value function.
We consider a policy that maps its previous action to a new action.
In this case, there is no external environment providing observations.

The value function, which we restrict to be linear, is tasked with predicting the future behaviour of the policy iteratively and online.
At the first timestep, we randomly initialize the parameters of the policy $\theta_{0} \sim p(\theta)$ and of the value function, $\mathbf{W}_{0} \sim p(\mathbf{W})$, using standard distributions for neural network initialization.
We then randomly sample the initial action, $b_{0} \sim N(0, 1/d)$, where $d=1000$ denotes the dimensionality of the action.
The function approximator, $\pi_{\theta}$ is then trained to maximize its interactivity whereas the value function $v_{\mathbf{W}}$ is trained using TD($0$).
Specifically, the following steps are repeated at each timestep $t$,
\begin{itemize}
  \item A trajectory of $T$ actions is produced by the policy, $\pi_{\theta_t}$.
  \item A copy of the current value function is frozen, and we record the temporal difference errors incurred along this trajectory: $\{\delta_{t,\gamma}(z_t, \theta_t), \dotso, \delta_{t+T,\gamma}(z_{t}, \theta_t)\}$. The sum of squared temporal difference errors provides a measure of the unconditional algorithmic complexity, $\mathbb{K}_{\mathcal{A}}(b_{t:T}|\epsilon) = \sum_{k=1}^T\delta_{t+k,\gamma}^2(z_{t}, \theta_t)$.

  \item A copy of the current value function is updated dynamically along the trajectory, and we record the temporal difference errors incurred along this trajectory: $\{\delta_{t,\gamma}(z_t, \theta_t), \dotso, \delta_{t+T,\gamma}(z_{t+T},\theta)\}$. Where each $z_{t+k}$ corresponds to the updated parameters of the value function. The sum of squared temporal difference errors provides a measure of the conditional algorithmic complexity, $\mathbb{K}_{\mathcal{A}}(b_{t:t+T-1}|b_{0:t-1}) = \sum_{k=1}^T\delta_{t+k,\gamma}^2(z_{t+k}, \theta_t)$.

  \item We update the policy using a single step from a gradient-based optimizer on the loss,
$$J(\theta) = \sum_{k=1}^T \delta_{t+k,\gamma}^2(z_{t}, \theta) - \delta_{t+k,\gamma}^2(z_{t+k}, \theta).$$

\item We take an action using the updated policy, $b_{t+1} = \pi(b_t; \theta_{t+1})$.
\item We now update the value function with a single step of TD($0$) using the temporal difference error, $\delta_t = b_{t+1} + \gamma v(b_{t+1}, \mathbf{W}_t) - v(b_{t}, \mathbf{W}_t)$.
\end{itemize}

\subsection{Interpreting Interactivity Maximization As A Continual Learning Benchmark}
Our theoretical and experimental results show that maximizing interactivity requires both fast adaptation to increase complexity with higher prediction errors and stability to sustain adaptation over time.
That is, maximizing interactivity involves the canonical plasticity-stability trade-off of continual learning \citep{grossberg82_how,parisi19_contin}.
Our empirical results demonstrate that deep nonlinear baselines fail at striking this balance, whereas deep linear networks appear to naturally achieve this balance.
This suggests that this synthetic benchmark isolates the key challenge in continual learning, while also not requiring outside data or environments.
This is significant because few environments are designed specifically to evaluate continual adaptation.
Thus, it is suboptimal to stop learning in this setting, regardless of the capacity of the algorithm or function approximator.

\subsection{Limitations of Experiments}

Our experiments used relatively shallow networks, with a maximum depth of $D=4$.
However, with the meta-gradient calculation over a finite horizon of $T=10$, the effective depth of the networks during auto-differentiation is $T\cdot D = 40$.
Meta-gradient methods for deep networks can exhibit more pathological learning dynamics due to increased curvature, leading to instability that could partially explain the discrepancy between linear and nonlinear networks.
Understanding how to control curvature using only first-order methods is key for effective meta-gradient descent in this setting.

The meta-gradient method poses several limitations in scaling.
Ideally, we would prefer to scale the horizon and the capacity of the function approximator.
However, because meta-gradients involve a second-order terms, involving the hessian, and because the horizon is multiplicative with the depth of the network, we have a computational complexity on the order $O(HD^2)$, where $D$ is the depth of the network.
Scaling both the horizon and the capacity results in a effective cubic scaling.
A more promising direction involves bootstrapping meta-gradients \citep{flennerhag22_boots_meta_learn}, and other first-order approximation \citep{nichol18_first_order_meta_learn_algor}.

Experimental evaluation in this setting also requires special consideration.
Holding the agent fixed for evaluation, as is commonly done in machine learning, is not appropriate given that interactivity is defined as an online objective.
In addition, standard approaches to hyperparameter tuning may not be feasible for evaluating the long-term performance of a continual learning agent \citep{mesbahi25_posit}.
Overcoming these obstacles to more fairly assess dependence on hyperparameters would require re-evaluation of several components of empirical practice in machine learning.

%%% Local Variables:
%%% TeX-master: "../neurips.tex"
%%% TeX-command-extra-options: "-shell-escape"
%%% End:

\section{Additional Background and Related Work}

\subsection{Algorithmic complexity}

The Kolmogorov complexity~\citep{Kolmogorov:65, Solomonoff:64, Chaitin:66} of an object (encoded as a binary string) is the length of the shortest program that computes it and halts.
Unlike traditional information theory, it measures the complexity of an individual object without depending on a stochastic source or ensemble.

The Kolmogorov complexity of a string depends on the choice of a universal Turing machine.
However, since any universal Turing machine can simulate another (e.g., via a compiler), the choice of the machine affects the Kolmogorov complexity by, at most, an additive constant independent of the specific string~\citep{li19_introd_kolmog_compl_its_applic_edition}.

Kolmogorov complexity is closely tied to compression, where the shortest description represents the most efficient compression for the given universal Turing machine.
Although Kolmogorov complexity is uncomputable, it is possible to compute improving upper bounds by searching over all possible programs in parallel and tracking the shortest candidate that generates the target string~\citep{li19_introd_kolmog_compl_its_applic_edition}.

\subsection{AIXI}

AIXI defines a general Bayes-optimal reinforcement learning agent in an unknown computable environment~\citep{Hutter:05uaibook}.
In this framework, the environment is represented by a Turing machine with unidirectional input and output tapes, and bidirectional working/internal tapes.
The agent's actions are received by the environment on its input tape, based on which it can write a computable history-based reward and observation on its output tape.

The AIXI agent acts in a Bayes-optimal manner by planning based on a posterior estimate over all computable environments, using Solomonoff’s universal prior as a starting point~\citep{Solomonoff:64}.
This prior assigns higher probability to `simpler' environments--those with lower Kolmogorov complexity.
However, both Solomonoff’s prior and AIXI are incomputable, making the development of practical approximations within this framework a key area of interest~\citep{VenessNHUS11}.

\subsection{Connections to intrinsic motivation and the free energy principle}

Previous work has explored several intrinsic drives that can guide agent behaviour without the need for explicit external rewards \citep{Schmidhuber:10ieeetamd, barto2013intrinsic}.
Many approaches to intrinsic motivation are developed within the framework of traditional RL, where the agents are not constrained relative to the environment.
As a result, these approaches may not be well-suited to a big world.
Nevertheless, interactivity shares connections to ideas such as mutual information maximization in intrinsic motivation.

The information gain of a dynamics model can serve as an intrinsic or auxiliary reward, promoting curious exploration \citep{Storck:95, houthooft2016vime}
Unlike curiosity driven by information gain, the goal of interactivity is not to learn an accurate model of the world.

Another related concept is Empowerment~\citep{klyubin2005empowerment}, where an agent seeks to maximize its control over its environment.
Empowerment-seeking agents aim to maximize the mutual information between their actions and future states.
Such agents avoid states where their actions have low influence and prefer states that allow for a wide range of controllable outcomes.
This objective can also be used to learn a set of behaviours (or options) that lead to different final states~\citep{mohamed2015variational, gregor2016variational}.
As discussed earlier, interactivity-maximizing agents produce complex yet predictable behaviour, which is not directly tied to the concept of control.
Furthermore, unlike objectives grounded in traditional (Shannon) information theory, interactivity relies on asymmetric algorithmic mutual information between previous inputs and future outputs.

Active inference describes agentic behavior in partially observable environments as the minimization of free energy~\citep{friston2010action, sajid2021active}.
Free-energy  minimization prefers selecting actions that lead to highly predictable states—inputs that are unsurprising to the agent's model.
In contrast to free-energy minimization, maximizing interactivity actively discourages low-complexity predictable states.

\subsection{Relationship to other notions of boundaries and Markov processes} 

The boundaried Markov process that we define is also similar to other frameworks which introduce explicit boundaries.
Active inference, for example, considers a Markov blanket that separates probabilistic nodes in a directed acyclic graph as a way of separating the agent from its environment \citep{kirchhoff2018markov}.
Open Markov processes have also been defined which have explicit boundary states from which probability can flow in and out of \citep{baez16_markov}.
Our work applies similar ideas specifically in the computational setting that we consider.

%%% Local Variables:
%%% TeX-master: "../neurips.tex"
%%% TeX-command-extra-options: "-shell-escape"
%%% End:

\newpage

\newpage
\section*{NeurIPS Paper Checklist}

\begin{enumerate}

\item {\bf Claims}
    \item[] Question: Do the main claims made in the abstract and introduction accurately reflect the paper's contributions and scope?
    \item[] Answer: \answerYes{}
    \item[] Justification:
    \item[] Guidelines:
    \begin{itemize}
        \item The answer NA means that the abstract and introduction do not include the claims made in the paper.
        \item The abstract and/or introduction should clearly state the claims made, including the contributions made in the paper and important assumptions and limitations. A No or NA answer to this question will not be perceived well by the reviewers.
        \item The claims made should match theoretical and experimental results, and reflect how much the results can be expected to generalize to other settings.
        \item It is fine to include aspirational goals as motivation as long as it is clear that these goals are not attained by the paper.
    \end{itemize}

\item {\bf Limitations}
    \item[] Question: Does the paper discuss the limitations of the work performed by the authors?
    \item[] Answer: \answerYes{} % Replace by \answerYes{}, \answerNo{}, or \answerNA{}.
    \item[] Justification: In the discussion, the limitations of an efficient algorithm for maximizing ineractivity is discussed.
    \item[] Guidelines:
    \begin{itemize}
        \item The answer NA means that the paper has no limitation while the answer No means that the paper has limitations, but those are not discussed in the paper.
        \item The authors are encouraged to create a separate "Limitations" section in their paper.
        \item The paper should point out any strong assumptions and how robust the results are to violations of these assumptions (e.g., independence assumptions, noiseless settings, model well-specification, asymptotic approximations only holding locally). The authors should reflect on how these assumptions might be violated in practice and what the implications would be.
        \item The authors should reflect on the scope of the claims made, e.g., if the approach was only tested on a few datasets or with a few runs. In general, empirical results often depend on implicit assumptions, which should be articulated.
        \item The authors should reflect on the factors that influence the performance of the approach. For example, a facial recognition algorithm may perform poorly when image resolution is low or images are taken in low lighting. Or a speech-to-text system might not be used reliably to provide closed captions for online lectures because it fails to handle technical jargon.
        \item The authors should discuss the computational efficiency of the proposed algorithms and how they scale with dataset size.
        \item If applicable, the authors should discuss possible limitations of their approach to address problems of privacy and fairness.
        \item While the authors might fear that complete honesty about limitations might be used by reviewers as grounds for rejection, a worse outcome might be that reviewers discover limitations that aren't acknowledged in the paper. The authors should use their best judgment and recognize that individual actions in favor of transparency play an important role in developing norms that preserve the integrity of the community. Reviewers will be specifically instructed to not penalize honesty concerning limitations.
    \end{itemize}

\item {\bf Theory assumptions and proofs}
    \item[] Question: For each theoretical result, does the paper provide the full set of assumptions and a complete (and correct) proof?
    \item[] Answer: \answerYes{}
    \item[] Justification:  Yes, in the appendix.
    \item[] Guidelines:
    \begin{itemize}
        \item The answer NA means that the paper does not include theoretical results.
        \item All the theorems, formulas, and proofs in the paper should be numbered and cross-referenced.
        \item All assumptions should be clearly stated or referenced in the statement of any theorems.
        \item The proofs can either appear in the main paper or the supplemental material, but if they appear in the supplemental material, the authors are encouraged to provide a short proof sketch to provide intuition.
        \item Inversely, any informal proof provided in the core of the paper should be complemented by formal proofs provided in appendix or supplemental material.
        \item Theorems and Lemmas that the proof relies upon should be properly referenced.
    \end{itemize}

    \item {\bf Experimental result reproducibility}
    \item[] Question: Does the paper fully disclose all the information needed to reproduce the main experimental results of the paper to the extent that it affects the main claims and/or conclusions of the paper (regardless of whether the code and data are provided or not)?
    \item[] Answer: \answerYes{}
    \item[] Justification: Yes, in the appendix.
    \item[] Guidelines:
    \begin{itemize}
        \item The answer NA means that the paper does not include experiments.
        \item If the paper includes experiments, a No answer to this question will not be perceived well by the reviewers: Making the paper reproducible is important, regardless of whether the code and data are provided or not.
        \item If the contribution is a dataset and/or model, the authors should describe the steps taken to make their results reproducible or verifiable.
        \item Depending on the contribution, reproducibility can be accomplished in various ways. For example, if the contribution is a novel architecture, describing the architecture fully might suffice, or if the contribution is a specific model and empirical evaluation, it may be necessary to either make it possible for others to replicate the model with the same dataset, or provide access to the model. In general. releasing code and data is often one good way to accomplish this, but reproducibility can also be provided via detailed instructions for how to replicate the results, access to a hosted model (e.g., in the case of a large language model), releasing of a model checkpoint, or other means that are appropriate to the research performed.
        \item While NeurIPS does not require releasing code, the conference does require all submissions to provide some reasonable avenue for reproducibility, which may depend on the nature of the contribution. For example
        \begin{enumerate}
            \item If the contribution is primarily a new algorithm, the paper should make it clear how to reproduce that algorithm.
            \item If the contribution is primarily a new model architecture, the paper should describe the architecture clearly and fully.
            \item If the contribution is a new model (e.g., a large language model), then there should either be a way to access this model for reproducing the results or a way to reproduce the model (e.g., with an open-source dataset or instructions for how to construct the dataset).
            \item We recognize that reproducibility may be tricky in some cases, in which case authors are welcome to describe the particular way they provide for reproducibility. In the case of closed-source models, it may be that access to the model is limited in some way (e.g., to registered users), but it should be possible for other researchers to have some path to reproducing or verifying the results.
        \end{enumerate}
    \end{itemize}

\item {\bf Open access to data and code}
    \item[] Question: Does the paper provide open access to the data and code, with sufficient instructions to faithfully reproduce the main experimental results, as described in supplemental material?
    \item[] Answer: \answerNA{}
    \item[] Justification: Not currently in the supplementary material, but we plan to open source it.
    \item[] Guidelines:
    \begin{itemize}
        \item The answer NA means that paper does not include experiments requiring code.
        \item Please see the NeurIPS code and data submission guidelines (\url{https://nips.cc/public/guides/CodeSubmissionPolicy}) for more details.
        \item While we encourage the release of code and data, we understand that this might not be possible, so “No” is an acceptable answer. Papers cannot be rejected simply for not including code, unless this is central to the contribution (e.g., for a new open-source benchmark).
        \item The instructions should contain the exact command and environment needed to run to reproduce the results. See the NeurIPS code and data submission guidelines (\url{https://nips.cc/public/guides/CodeSubmissionPolicy}) for more details.
        \item The authors should provide instructions on data access and preparation, including how to access the raw data, preprocessed data, intermediate data, and generated data, etc.
        \item The authors should provide scripts to reproduce all experimental results for the new proposed method and baselines. If only a subset of experiments are reproducible, they should state which ones are omitted from the script and why.
        \item At submission time, to preserve anonymity, the authors should release anonymized versions (if applicable).
        \item Providing as much information as possible in supplemental material (appended to the paper) is recommended, but including URLs to data and code is permitted.
    \end{itemize}

\item {\bf Experimental setting/details}
    \item[] Question: Does the paper specify all the training and test details (e.g., data splits, hyperparameters, how they were chosen, type of optimizer, etc.) necessary to understand the results?
    \item[] Answer: \answerYes{}
    \item[] Justification: In the appendix and supplementary material.
    \item[] Guidelines:
    \begin{itemize}
        \item The answer NA means that the paper does not include experiments.
        \item The experimental setting should be presented in the core of the paper to a level of detail that is necessary to appreciate the results and make sense of them.
        \item The full details can be provided either with the code, in appendix, or as supplemental material.
    \end{itemize}

\item {\bf Experiment statistical significance}
    \item[] Question: Does the paper report error bars suitably and correctly defined or other appropriate information about the statistical significance of the experiments?
    \item[] Answer: \answerYes{}
    \item[] Justification: Shaded error bars are provided, but no statistical p-test.
    \item[] Guidelines:
    \begin{itemize}
        \item The answer NA means that the paper does not include experiments.
        \item The authors should answer "Yes" if the results are accompanied by error bars, confidence intervals, or statistical significance tests, at least for the experiments that support the main claims of the paper.
        \item The factors of variability that the error bars are capturing should be clearly stated (for example, train/test split, initialization, random drawing of some parameter, or overall run with given experimental conditions).
        \item The method for calculating the error bars should be explained (closed form formula, call to a library function, bootstrap, etc.)
        \item The assumptions made should be given (e.g., Normally distributed errors).
        \item It should be clear whether the error bar is the standard deviation or the standard error of the mean.
        \item It is OK to report 1-sigma error bars, but one should state it. The authors should preferably report a 2-sigma error bar than state that they have a 96\% CI, if the hypothesis of Normality of errors is not verified.
        \item For asymmetric distributions, the authors should be careful not to show in tables or figures symmetric error bars that would yield results that are out of range (e.g. negative error rates).
        \item If error bars are reported in tables or plots, The authors should explain in the text how they were calculated and reference the corresponding figures or tables in the text.
    \end{itemize}

\item {\bf Experiments compute resources}
    \item[] Question: For each experiment, does the paper provide sufficient information on the computer resources (type of compute workers, memory, time of execution) needed to reproduce the experiments?
    \item[] Answer: \answerNo{}
    \item[] Justification: No, only a local cluster was used which did not log computational resources, but they can be replicated on a modern GPU in less than 24 hours.
    \item[] Guidelines:
    \begin{itemize}
        \item The answer NA means that the paper does not include experiments.
        \item The paper should indicate the type of compute workers CPU or GPU, internal cluster, or cloud provider, including relevant memory and storage.
        \item The paper should provide the amount of compute required for each of the individual experimental runs as well as estimate the total compute.
        \item The paper should disclose whether the full research project required more compute than the experiments reported in the paper (e.g., preliminary or failed experiments that didn't make it into the paper).
    \end{itemize}

\item {\bf Code of ethics}
    \item[] Question: Does the research conducted in the paper conform, in every respect, with the NeurIPS Code of Ethics \url{https://neurips.cc/public/EthicsGuidelines}?
    \item[] Answer: \answerYes{}
    \item[] Justification:
    \item[] Guidelines:
    \begin{itemize}
        \item The answer NA means that the authors have not reviewed the NeurIPS Code of Ethics.
        \item If the authors answer No, they should explain the special circumstances that require a deviation from the Code of Ethics.
        \item The authors should make sure to preserve anonymity (e.g., if there is a special consideration due to laws or regulations in their jurisdiction).
    \end{itemize}

\item {\bf Broader impacts}
    \item[] Question: Does the paper discuss both potential positive societal impacts and negative societal impacts of the work performed?
    \item[] Answer: \answerNo{}
    \item[] Justification: It does not, as it is purely a theoretical investigation.
    \item[] Guidelines:
    \begin{itemize}
        \item The answer NA means that there is no societal impact of the work performed.
        \item If the authors answer NA or No, they should explain why their work has no societal impact or why the paper does not address societal impact.
        \item Examples of negative societal impacts include potential malicious or unintended uses (e.g., disinformation, generating fake profiles, surveillance), fairness considerations (e.g., deployment of technologies that could make decisions that unfairly impact specific groups), privacy considerations, and security considerations.
        \item The conference expects that many papers will be foundational research and not tied to particular applications, let alone deployments. However, if there is a direct path to any negative applications, the authors should point it out. For example, it is legitimate to point out that an improvement in the quality of generative models could be used to generate deepfakes for disinformation. On the other hand, it is not needed to point out that a generic algorithm for optimizing neural networks could enable people to train models that generate Deepfakes faster.
        \item The authors should consider possible harms that could arise when the technology is being used as intended and functioning correctly, harms that could arise when the technology is being used as intended but gives incorrect results, and harms following from (intentional or unintentional) misuse of the technology.
        \item If there are negative societal impacts, the authors could also discuss possible mitigation strategies (e.g., gated release of models, providing defenses in addition to attacks, mechanisms for monitoring misuse, mechanisms to monitor how a system learns from feedback over time, improving the efficiency and accessibility of ML).
    \end{itemize}

\item {\bf Safeguards}
    \item[] Question: Does the paper describe safeguards that have been put in place for responsible release of data or models that have a high risk for misuse (e.g., pretrained language models, image generators, or scraped datasets)?
    \item[] Answer: \answerNA{}
    \item[] Justification: None of these were used.
    \item[] Guidelines:
    \begin{itemize}
        \item The answer NA means that the paper poses no such risks.
        \item Released models that have a high risk for misuse or dual-use should be released with necessary safeguards to allow for controlled use of the model, for example by requiring that users adhere to usage guidelines or restrictions to access the model or implementing safety filters.
        \item Datasets that have been scraped from the Internet could pose safety risks. The authors should describe how they avoided releasing unsafe images.
        \item We recognize that providing effective safeguards is challenging, and many papers do not require this, but we encourage authors to take this into account and make a best faith effort.
    \end{itemize}

\item {\bf Licenses for existing assets}
    \item[] Question: Are the creators or original owners of assets (e.g., code, data, models), used in the paper, properly credited and are the license and terms of use explicitly mentioned and properly respected?
    \item[] Answer: \answerNA{}
    \item[] Justification: Not used.
    \item[] Guidelines:
    \begin{itemize}
        \item The answer NA means that the paper does not use existing assets.
        \item The authors should cite the original paper that produced the code package or dataset.
        \item The authors should state which version of the asset is used and, if possible, include a URL.
        \item The name of the license (e.g., CC-BY 4.0) should be included for each asset.
        \item For scraped data from a particular source (e.g., website), the copyright and terms of service of that source should be provided.
        \item If assets are released, the license, copyright information, and terms of use in the package should be provided. For popular datasets, \url{paperswithcode.com/datasets} has curated licenses for some datasets. Their licensing guide can help determine the license of a dataset.
        \item For existing datasets that are re-packaged, both the original license and the license of the derived asset (if it has changed) should be provided.
        \item If this information is not available online, the authors are encouraged to reach out to the asset's creators.
    \end{itemize}

\item {\bf New assets}
    \item[] Question: Are new assets introduced in the paper well documented and is the documentation provided alongside the assets?
    \item[] Answer: \answerNA{}
    \item[] Justification:
    \item[] Guidelines:
    \begin{itemize}
        \item The answer NA means that the paper does not release new assets.
        \item Researchers should communicate the details of the dataset/code/model as part of their submissions via structured templates. This includes details about training, license, limitations, etc.
        \item The paper should discuss whether and how consent was obtained from people whose asset is used.
        \item At submission time, remember to anonymize your assets (if applicable). You can either create an anonymized URL or include an anonymized zip file.
    \end{itemize}

\item {\bf Crowdsourcing and research with human subjects}
    \item[] Question: For crowdsourcing experiments and research with human subjects, does the paper include the full text of instructions given to participants and screenshots, if applicable, as well as details about compensation (if any)?
    \item[] Answer: \answerNA{}
    \item[] Justification:
    \item[] Guidelines:
    \begin{itemize}
        \item The answer NA means that the paper does not involve crowdsourcing nor research with human subjects.
        \item Including this information in the supplemental material is fine, but if the main contribution of the paper involves human subjects, then as much detail as possible should be included in the main paper.
        \item According to the NeurIPS Code of Ethics, workers involved in data collection, curation, or other labor should be paid at least the minimum wage in the country of the data collector.
    \end{itemize}

\item {\bf Institutional review board (IRB) approvals or equivalent for research with human subjects}
    \item[] Question: Does the paper describe potential risks incurred by study participants, whether such risks were disclosed to the subjects, and whether Institutional Review Board (IRB) approvals (or an equivalent approval/review based on the requirements of your country or institution) were obtained?
    \item[] Answer: \answerNA{}
    \item[] Justification:
    \item[] Guidelines:
    \begin{itemize}
        \item The answer NA means that the paper does not involve crowdsourcing nor research with human subjects.
        \item Depending on the country in which research is conducted, IRB approval (or equivalent) may be required for any human subjects research. If you obtained IRB approval, you should clearly state this in the paper.
        \item We recognize that the procedures for this may vary significantly between institutions and locations, and we expect authors to adhere to the NeurIPS Code of Ethics and the guidelines for their institution.
        \item For initial submissions, do not include any information that would break anonymity (if applicable), such as the institution conducting the review.
    \end{itemize}

\item {\bf Declaration of LLM usage}
    \item[] Question: Does the paper describe the usage of LLMs if it is an important, original, or non-standard component of the core methods in this research? Note that if the LLM is used only for writing, editing, or formatting purposes and does not impact the core methodology, scientific rigorousness, or originality of the research, declaration is not required.
    %this research?
    \item[] Answer: \answerNo{}
    \item[] Justification:
    \item[] Guidelines:
    \begin{itemize}
        \item The answer NA means that the core method development in this research does not involve LLMs as any important, original, or non-standard components.
        \item Please refer to our LLM policy (\url{https://neurips.cc/Conferences/2025/LLM}) for what should or should not be described.
    \end{itemize}

\end{enumerate}

%%% Local Variables:
%%% TeX-master: "../neurips.tex"
%%% TeX-command-extra-options: "-shell-escape"
%%% End:

\end{document}